\documentclass[sigconf]{acmart}

\AtBeginDocument{%
  \providecommand\BibTeX{{%
    \normalfont B\kern-0.5em{\scshape i\kern-0.25em b}\kern-0.8em\TeX}}}

\DeclareMathOperator*{\argmax}{arg\,max}

\usepackage{subcaption}
\usepackage[linesnumbered,ruled]{algorithm2e}
\usepackage{graphicx}
\usepackage{amsmath,amsfonts}
\usepackage[mathscr]{eucal}
\usepackage{comment}
\usepackage{booktabs}
\usepackage{multirow}
\setcopyright{acmcopyright}
\copyrightyear{2022}
\acmYear{2022}
\acmDOI{10.1145/1122445.1122456}

\acmPrice{15.00}
\acmISBN{978-1-4503-XXXX-X/18/06}

\begin{document}
\newtheorem{dfn}{Definition}
\title[Large-Scale Privacy-Preserving Network Embedding against Private Link Inference Attacks]{Large-Scale Privacy-Preserving Network Embedding \\ against Private Link Inference Attacks}

\author{Xiao Han}
\affiliation{%
	\institution{Shanghai University of Finance and Economics, China}
}
\author{Leye Wang}
\affiliation{%
	\institution{Peking University, China}
}
\author{Junjie Wu}
\affiliation{%
	\institution{Beihang University, China}
}
\author{Yuncong Yang}
\affiliation{%
	\institution{Shanghai University of Finance and Economics, China}
}

%

\renewcommand{\shortauthors}{Han et al.}

\begin{abstract}
Network embedding represents network nodes by a low-dimensional informative vector. While it is generally effective for various downstream tasks, it may leak some private information of networks, such as hidden private links. In this work, we address a novel problem of \textit{privacy-preserving network embedding against private link inference attacks}. Basically, we propose to perturb the original network by adding or removing links, and expect the embedding generated on the perturbed network can leak little information about private links but hold high utility for various downstream tasks. Towards this goal, we first propose general measurements to quantify privacy gain and utility loss incurred by candidate network perturbations; we then design a PPNE framework to identify the optimal perturbation solution with the best privacy-utility trade-off in an iterative way. Furthermore, we propose many techniques to accelerate PPNE and ensure its scalability. For instance, as the skip-gram embedding methods including DeepWalk and LINE can be seen as matrix factorization with closed form embedding results, we devise efficient privacy gain and utility loss approximation methods to avoid the repetitive time-consuming embedding training for every candidate network perturbation in each iteration. 
Experiments on real-life network datasets (with up to millions of nodes) verify that PPNE outperforms baselines by sacrificing less utility and obtaining higher privacy protection. 
\end{abstract}

%


\maketitle

\section{Introduction}
Data publishing has become imperative for comprehensive data exploration and utilization~\cite{fung2010privacy, jia2018attriguard,yang2019privacy}. Network data such as citation networks, social networks, and communication networks, is ubiquitous in human life and useful for many applications~\cite{cheng2010k, chen2018disclose, yu2019target}. 
While network structure is directly published in tradition, network embedding has become a good substitution for publishing along the prosperity of the network embedding techniques.
In particular, network embedding techniques~\cite{goyal2018graph, hamilton2017representation} typically represent each node by a low-dimensional dense vector which can be effectively used for a variety of downstream tasks like node classification \cite{wang2016structural, tang2015line, ou2016asymmetric, perozzi2014deepwalk, grover2016node2vec}, and clustering \cite{white2005spectral}. 

However, releasing network embedding gives malicious attackers opportunities to infer the private information in the network, which may lead to severe privacy leakage problems~\cite{duddu2020quantifying}. For example, many methods (\emph{e.g.,} DeepWalk~\cite{perozzi2014deepwalk}, LINE~\cite{tang2015line}, and node2vec~\cite{grover2016node2vec}) learn network embedding based on the publicly available links between nodes; nevertheless, such embedding may be used by attackers to correctly infer the hidden private links in the network~\cite{grover2016node2vec}. Due to the data regulation and privacy law (e.g., GDPR\footnote{https://gdpr.eu/eu-gdpr-personal-data}), private information should be protected from being accessed during data publishing. Hence, a privacy-preserving network embedding publishing method is urgently needed to protect the private link of a network while publishing its embedding.


Recently, some pioneering studies begin to explore the problem of privacy-preserving network embedding for publishing with different privacy protection goals. Specifically, a differentially private network embedding method is proposed with an optimization objective that the perturbed network embedding cannot be differentiated from the original~\cite{xu2018dpne}; whereas, differential privacy protection methods cannot ideally defend against private link inference attacks \cite{jayaraman2019evaluating}. Some other studies learn privacy-preserving embedding by graph neural networks via adversarial training~\cite{li2020adversarial,pmlr-v139-liao21a,wang2021privacy}; they expect that the learned embedding is useful for one target task but ineffective to other tasks. In brief, \textit{there is still lack of privacy-preserving network embedding publishing methods that can simultaneously (i) prevent inference attacks on private links, and (ii) benefit various downstream tasks such as node classification and clustering.}

We intentionally tackle the privacy-preserving network embedding publishing problem to defend against private link inference attacks. Basically, we propose to perturb the network structure (\emph{i.e.,}  add or delete links) and publish the embedding of perturbed network; we expect that the private link inference attacks become ineffective with the embedding of a proper perturbed network. However, the network perturbation generally result in data utility degradation on downstream tasks as well. Hence, the key is how to use the minimal utility loss to trade off the maximal privacy protection effect. More specifically, to find the best network perturbation solution that can maximize privacy protection and minimize utility loss, there are the following challenges:

\textbf{Privacy and utility quantification}. Given a certain network perturbation solution,  it is non-trivial to quantify its impacts on the network data utility for downstream tasks and privacy protection against private link inference attacks. In particular, as the embedding generation is a complicated process, how a network perturbation would impact embedding and further decrease the correctness of private link inference is hard to quantify.
Furthermore, as the learned embedding needs to serve for various downstream tasks, it would be favorable to have a task-independent utility metric that can effectively assess the general value of embedding in various (unknown) downstream tasks.

\textbf{Scalability for large-scale networks}. In principle, there are $O(2^{|V|^2})$ possible network perturbation solutions, where $|V|$ is the number of network nodes. For a large-scale network, it is practically intractable to select the best solution considering the privacy-utility tradeoff by enumerating all the candidate perturbations. Therefore, an efficient method is required to help find the best network perturbation solution quickly for large-scale networks.



We address the challenges and make the following contributions:

(1) To the best of our knowledge, this is the first work to study the problem of \textit{privacy-preserving network embedding publishing against private link inference attacks}. 
	
(2) To solve the problem, we design a framework called PPNE which presents three key technical merits: (i) a privacy measurement to quantify the protection effect against private link inference attacks given a perturbed network, (ii) a utility measurement to quantify the performance of various downstream tasks given a perturbed network, and (iii) an iterative network perturbation process achieving the optimal privacy-utility tradeoff. Moreover, for popular skip-gram embedding methods such as DeepWalk \cite{perozzi2014deepwalk} and LINE \cite{tang2015line}, we propose several techniques to efficiently approximate the privacy and utility measurements, which significantly improve the scalability of PPNE.

(3) We evaluate PPNE with four real-life datasets (with up to millions of nodes) and two downstream tasks (node classification and clustering), considering the embedding methods of DeepWalk and LINE. The results show that PPNE significantly outperforms baseline methods in achieving better privacy-utility tradeoff. Specifically, PPNE can work for a million-node network in 3 hours with a normal server.

\section{Problem Formulation and Preliminaries}
\label{subsec:problem}
The structure of network (e.g., social networks) is one of the most important data sources for various downstream network analysis tasks, such as node classification and clustering. Instead of directly releasing the original network structure, data holders nowadays are more and more willing to publish network embedding for world-wide data exploration and exploitation\footnote{Some embeddings are published at \url{http://snap.stanford.edu/data/web-RedditEmbeddings.html} and \url{https://projector.tensorflow.org/}},  because of the favorable properties of network embedding such as low-dimensional, informative, and task-independent. This work particularly focuses on a privacy-preserving network embedding publishing problem for an undirected network that includes both public and private links.

\begin{definition}[An observed undirected network] Let $V=\{v_i\}$ denote all nodes of a network. An observed undirected network can be denoted by $G=\left( V, E \right)$, where $E$ includes the publicly available undirected links between nodes. The adjacency matrix of the observed network $G$ is denoted by $A = \{a_{ij}\}\in \{0,1\}^{|V| \times |V|}$; $a_{ij}=1$ when there is an observed link between $v_i$ and $v_j$, and 0 otherwise. 
\end{definition}

\begin{definition}[Private link] A private link $e_p\langle v_i,v_j \rangle \in E_p$ refers to the actually existing but publicly unobservable links between nodes $v_i, v_j \in V$ of the observed network $G$. In practice, a complete network may include some private links besides of the observed network. For instance, Facebook allows users to hide their private relationships from the public.
\end{definition}


\begin{definition}[Network embedding] A network embedding method $\mathcal M$ maps a network to a low-dimensional representation, denoted as $X \in \mathbb{R}^{\left| V \right| \times d}$, $d \ll \left| V \right|$. The $i$-th row of $X$ is the $d$-dimensional representation of node $v_i$, which can be taken as an input of machine learning models for various downstream tasks.
\end{definition}

While node embedding is generally learned on the observed network~$G$, it contains rich structural information among nodes and thus may be used by attackers to infer the private links $E_p$ among nodes $V$. In other words, the private links may be leaked due to the embedding-based private link inference attacks. 

\begin{definition}[Embedding-based private link inference attack] Let $\mathcal P \subset V \times V$ denotes the set of unconnected node pairs in the observed network, which consists of the pairs ($\mathcal P_{pos}$) with unobservable private links $E_p$ and the pairs $(\mathcal P_{neg})$ without links
. An embedding-based private link inference attack aims to differentiate $\mathcal P_{pos}$ and $\mathcal P_{neg}$ from $\mathcal P$ given the network embedding $X$.
\end{definition}

\begin{definition}[Network perturbation] A network perturbation is a typical solution against private link inference attack by perturbing (\emph{i.e.,} adding or deleting) certain links given a network. The perturbed adjacency matrix $A'$ is the adjacency matrix of a perturbed network and can represent a specific network perturbation method.
\end{definition}

This work aims to produce network embedding for publishing and meanwhile defend against embedding-based private link inference attack by network perturbation methods. Based on the above definition, we formulate our research problem.

\begin{definition}[\textbf{Privacy-preserving network embedding publishing problem}] Given an observed undirected network $G = (V,E)$ with the adjacency matrix $A$, and an embedding method $\mathcal M$, we aim to find the optimal perturbed adjacency matrix $A'$ (adding or removing links from $A$), so that the released embedding $X' = \mathcal M(A')$ can concurrently achieve the following two objectives:

(i) The objective of privacy: the attacker cannot infer $\mathcal P_{pos}$ and $\mathcal P_{neg}$ from $\mathcal P$ precisely;
\begin{equation}
\max \   \textit{error}_\textit{infer }(\mathcal P_{pos}, \mathcal P_{neg} | X', \mathcal P)
\end{equation}

(ii) The objective of utility: the embedding $X' = \mathcal M(A'$) is as good as the original embedding $X = \mathcal M(A)$ for various downstream tasks such as node classification and clustering.
\begin{equation}
\min \  utility(X)-utility(X')
\end{equation}
\end{definition}

\emph{Remark}. Although there is increasing interest in preserving privacy of a graph while exploring its public information for downstream tasks, our privacy-preserving network embedding publishing problem presents some novelties. Our network embedding publishing problem is more challenging than the traditional \textit{privacy-preserving original network publishing} issues~\cite{yu2019target,Chen2020DiscloseMA,XiaoKDD2014}, as it needs to deal with the complicated transformation from a graph to embedding when defending against private link inference attack. Besides, our work is different from  the recent privacy-preserving graph representation learning studies \cite{wang2021privacy,pmlr-v139-liao21a,li2020adversarial}. Specifically, the literature aims at learning representations for one task and preventing the learned representations from being effectively used for other tasks; comparatively, our goal is to produce task-independent network embedding that can be used for various downstream tasks, and meanwhile protect the private links from being inferred by the published network embedding. In the subsequent section, we will propose methods to address the novel and challenging \textit{privacy-preserving network embedding publishing problem}.

\color{black}
\section{Methodology}
Analytically, we have $O(2^{|V|^2})$ possible network perturbation methods given a network with $|V|$ nodes, since we could either keep or flip the link status of any node pairs (\emph{i.e.,} delete the existing link between two nodes or add one if there isn't). Enumerating the perturbations to find the optimum would be intractable for a practical large-scale network. In this vein, we propose an iterative solution (namely PPNE) to discover a proper perturbed network for privacy-preserving network embedding. In each iteration, PPNE estimates the privacy gain and utility loss once the link status of any node pair is flipped, and then selectively flips the best node pairs that can lead to the maximum privacy gain and minimum utility loss. Fig.~\ref{fig:overview} overviews PPNE which includes three main components: (1) the \textit{privacy gain} computation, (2) the \textit{utility loss} quantification, and (3) the optimal network perturbation. Next, we elaborate on each of the components in Sec.~\ref{sub:toy_ppne}; we then analyze the potential scalability issue, and enhances PPNE to be able to work for very large networks including millions of nodes in Sec.~\ref{sub:large_scale_ppne}.


\begin{figure}[t!]
	\centering
	\includegraphics[width=0.44\textwidth]{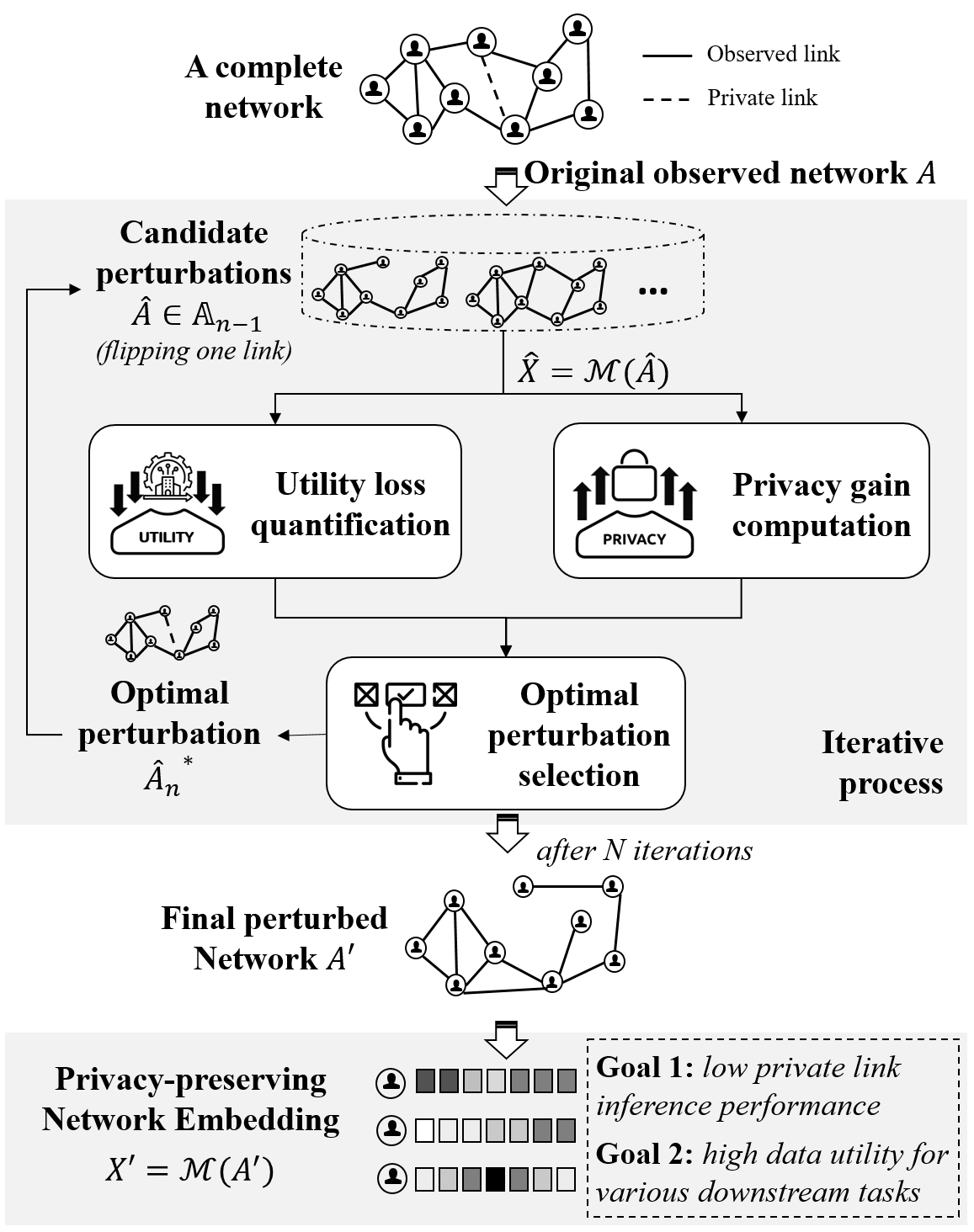}
	\vspace{-1em}
	\caption{Overview of PPNE.}
	\vspace{-1.5em}
	\label{fig:overview}
\end{figure}


\subsection{Core Design of PPNE}
\label{sub:toy_ppne}
\subsubsection{Privacy Gain Computation} 
\label{subsub:privacy_gain} Generally speaking, if network embedding $X'$ is published, attackers may perform embedding based inference attacks to correctly infer private links between nodes, which will incur privacy leakage. To compute the privacy leakage,
we first assume that the attackers only obtain the published network embedding $X'$ and will discuss the cases where they acquire some auxiliary knowledge at the end of this section. In such a context, the released embedding $\textbf{x}'_i \in X'$ of node $v_i$ can be seen as its feature vector, and the similarity-based link inference is the de-facto attack method~\cite{zhou2009predicting,yu2019target,lu2011link}. 
More specifically, the attackers compute the similarity score (\emph{i.e.,} $sim(v_i,v_j)$) of node pairs $(v_i,v_j) \in \mathcal P$ by the inner product of nodes' embedding\footnote{The inner product of nodes' normalized embedding equals their \textit{cosine} similarity, which is one of the most widely used similarity metric~\cite{levy2015improving}}, and accordingly decide whether a link exists or not between $v_i$ and $v_j$. The larger the similarity score $sim(v_i,v_j)$ is, the higher probability the link exists. In this light, we can model the embedding-based private link inference attack as:
\begin{equation}
	\textit{Attack}(v_i,v_j) = sim(\textbf{x}'_i, \textbf{x}'_j)= \textbf{x}'_i {\textbf{x}'_j}^T, \quad \forall (v_i,v_j) \in \mathcal P
	\label{eq:attack_I_def}
\end{equation}

In principle, the privacy leakage is associated with the correctness of attackers' inference. If the embedding similarities of node pairs with a private link (\emph{i.e.,} $(v_i, v_j) \in \mathcal{P}_{pos}$) are high and the value of node pairs without a private link (\emph{i.e.,} $(v_i, v_j) \in \mathcal{P}_{neg}$) are low, the attackers can easily attain a high inference correctness by using the similarity-based attack method and thus the privacy leakage would be severe; vice-versus. In this vein, we measure the privacy leakage due to the released embedding $X'$ by the embedding similarities of all node pairs in $\mathcal{P}_{pos}$ and $\mathcal{P}_{neg}$ as,
\begin{equation}
	PL(X') = \sum_{(v_i,v_j) \in \mathcal P_{pos}} \textbf{x}'_i {\textbf{x}'_j}^T - \sum_{(v_i,v_j) \in \mathcal P_{neg}} \textbf{x}'_i {\textbf{x}'_j}^T
	\label{eq:privacy_leakage}
\end{equation}
where $\sum_{(v_i,v_j) \in \mathcal P_{pos}} \textbf{x}'_i {\textbf{x}'_j}^T$ is the sum of embedding similarities between nodes in $\mathcal{P}_{pos}$, $\sum_{(v_i,v_j) \in \mathcal P_{neg}} \textbf{x}'_i {\textbf{x}'_j}^T$ aggregates the embedding similarities between nodes in $\mathcal{P}_{neg}$, and the privacy leakage $PL(X')$ is lower if the similarities of node pairs in $\mathcal{P}_{pos}$ are lower and the similarities of node pairs in $\mathcal{P}_{neg}$ are higher.

To defend against private link inference attacks, we perturb the original network $A$ and obtain the embedding $X'$ of the perturbed network $A'$ to publish. The privacy leakage is expected to be reduced by publishing $X'$ instead of releasing the embedding $X$ of original network $A$. We compute the privacy gain of publishing $X'$ over $X$ by the reduced privacy leakage as,
\begin{equation}
	PG(X'|X) =  PL(X) - PL(X')
\end{equation}

\textit{Remark}. Most of the prior related works only consider the attacks with no auxiliary knowledge~\cite{yu2019target}. Although attacks with auxiliary knowledge are also common in practice, it is hard to foreknow the auxiliary knowledge of an attacker, and unfeasible to defend against exhaustive attackers with various auxiliary knowledge. 
In this work, we note that the attacks with no auxiliary knowledge can be seen as a special case of attacks with certain auxiliary knowledge, and show that the proposed method for the attacks with no auxiliary knowledge is also effective for general attacks with auxiliary knowledge to some extent. Some detailed theoretical proofs and empirical results are presented in Appendix.

\subsubsection{Utility Loss Quantification} As the published network embedding may be used for various downstream tasks, a task-independent utility loss metric would be favorable. In general, the task-independent data utility loss caused by data alterations is measured by the expected difference between the altered data and original data. In our context, the original data is the network $A$ and the altered publishing data is the perturbed embedding $X'$. There is a data utility loss when releasing the perturbed embedding $X'$ instead of releasing $A$. We accordingly quantify the utility loss by the data difference between $X'$ and $A$, denoted as $UL(X'|A)$. However, it is very difficult to directly compute $UL(X'|A)$ by a general distance function, as $X'$ and $A$ are not in the same dimensional space. Thus we further decompose the utility loss quantification according to the chain-rule as:
\begin{equation}
	UL(X'|A) = UL(X'|X) + UL(X|A)
\end{equation}
where $UL(X'|X)$ and $UL(X|A)$ represent the data differences of $X$ to $X'$ and $A$ respectively. Fortunately, $UL(X|A)$ is a constant for a given network embedding method $\mathcal M$. Then, the goal of minimizing the data utility loss $UL(X'|A)$ is equivalent to minimizing $UL(X'|X)$, which can be quantified by:
\begin{equation}
	UL(X'|X) = d(X', X)
\end{equation}
where $d$ can be any widely used distance function such as root mean squared error.

\subsubsection{Iterative Network Perturbation} With the utility loss and privacy gain quantification methods, we describe how we iteratively identify proper link status of node pairs in the original network $A$ to perturb and finally derive the perturbed $A'$. Ideally, the embedding $X'$ generated from the perturbed network $A'$ can lead to the maximum privacy gain and minimum utility loss. Therefore, we integrate these two objectives to select the best link for perturbation and obtain the best perturbed network $\hat A^*_n$ in the $n$-th iteration as,
\begin{equation}
	\hat A^*_n = \argmax_{\hat A \in \mathbb{A}_n} \frac{PG(\hat X|X)}{[UL(\hat X|X)]^k} = \argmax_{\hat A \in \mathbb{A}_n} \frac{PG(\mathcal M(\hat A)|\mathcal M(A))}{[UL(\mathcal M(\hat A)|\mathcal M(A))]^k}
	\label{eq:pu_tradeoff}
\end{equation}
where $\hat A^*_{0} = A$, and $k$ is a parameter to tune the relative importance of utility over privacy in the selection criteria. To generate the candidate perturbed networks $\mathbb{A}_n$ for the $n$-th iteration, we only flip one link status of any unperturbed pair nodes given the best perturbed network of the previous iteration $\hat A^*_{n-1}$ for the computational simplicity and efficiency. That is saying $||\hat A-\hat A^*_{n-1}||=1, \forall \hat A \in \mathbb{A}_n$.

Algorithm~\ref{alg:ppne} illustrates the iterative network perturbation process of PPNE. It stops and outputs the final perturbed matrix $A' = \hat A^*_N$ when it achieves a sufficient number of iterations or an expected privacy gain (\emph{i.e.,} $PG(\mathcal M(A^*_N)|\mathcal M(A))$). The embedding $X'$ of $A'$ is published as the privacy-preserving network embedding. 

\begin{algorithm}[t]
	\SetKwInOut{Input}{Input}
	\SetKwInOut{Output}{Output}
	\Input{
		$A$: original network;
		$\mathcal M$: embedding method.
	}
	\Output{$X'$: privacy-preserving network embedding.}
	$\hat A^*_0 \leftarrow A$ \;
	$X = \mathcal M(A)$\;
	\For{$n \in \{1 ... N\}$}{
		
		$\mathbb A_n \leftarrow  \{\hat A|\ ||\hat A-\hat A^*_{n-1}||=1\}$\;
		\For{$\hat A \in \mathbb A_n$}{
		$\hat X \leftarrow \mathcal M(\hat A)$\; \label{line:embedding_recalculation}
		$PG(\hat X|X) \leftarrow  PL(X) - PL(\hat X)$\;
		$UL(\hat X|X) \leftarrow d(\hat X, X)$\;

		}
		$\hat A^*_n \leftarrow \argmax_{\hat A \in \mathbb{A}_n} \frac{PG(\hat X|X)}{[UL(\hat X|X)]^k}$
	}
	$X' \leftarrow \mathcal M(\hat A^*_N)$
	\caption{Core Design of PPNE}
	\label{alg:ppne}
\end{algorithm}


\subsection{Toward A Million-Node Network}
\label{sub:large_scale_ppne}

So far, we have highlighted the core components in PPNE. However, the scalability of Algorithm~\ref{alg:ppne} is limited. Specifically, for each iteration, we need to train the embedding $\hat X$ for every perturbation $\hat A \in \mathbb A_n$ (line~\ref{line:embedding_recalculation} in Algorithm~\ref{alg:ppne}). Since training embedding for one large network is often time-consuming, it is intractable to train $\hat X$ for all possible $\hat A$. For example, with a network including 10,000 nodes, each iteration needs to train embedding for around 100,000,000 times to enumerate all possible $\hat A$; this is practically impossible. Here, with the core design considerations in mind, we propose several techniques to enhance the scalability of PPNE, so that it can work for a large-scale network with millions of nodes.


\subsubsection{Training-free Privacy Gain Computation} Recall that the privacy gain in each iteration is modeled as,
\begin{equation}
	PG(\hat X|X) =  PL(X) - PL(\hat X)
\end{equation}
Concerning that $PL(X)$ is a fixed value given a certain network, thus only $PL(\hat X)=PL(\mathcal M(\hat A))$ is necessarily computed to determine $PG(\hat X|X)$. As aforementioned, it would be quite time-consuming to train the embedding $\mathcal M(\hat A)$ for each $\hat A \in \mathbb A_n$; therefore the computational efficiency of $PL(\mathcal M(\hat A))$ is  severely limited. What's worse, PPNE needs to compute $PL(\mathcal M(\hat A))$ for a huge number of candidate perturbation matrices in each iteration. Therefore, we intend to accelerate the efficiency of $PL(\mathcal M(\hat A))$ computation.

Note that, for each $\hat A \in \mathbb A_n$, there is a subtle change (\emph{i.e.,} only one element difference) from $\hat A^*_{n-1}$. Inspired by \citet{sun2018data}, if we could remove the embedding training process for every $\hat A \in \mathbb A$, and directly estimate $PL(\mathcal M(\hat A))$ from $PL(\mathcal M(\hat A^*_{n-1}))$ 
based on the subtle change between $\hat A$ and $\hat A^*_{n-1}$, the computational time would be dramatically reduced. Based on this idea, we design a \textit{privacy-oriented gradient} (POG) method to deduce $PL(\mathcal M(\hat A))$ for every $\hat A \in \mathbb A_n$ given $\hat A^*_{n-1}$ and $PL(\mathcal M(\hat A^*_{n-1}))$.



Specifically, we use $\nabla_{a_{ij}} PL(\mathcal M(\hat A^*_{n-1}))$ to denote the difference of $PL(\mathcal M(\hat A))$ to $PL(\mathcal M(\hat A^*_{n-1}))$,  where $a_{ij}$ is the only different element between $\hat A$ and $\hat A^*_{n-1}$. Then, we have
\begin{equation}
	PL(\mathcal M (\hat A)) = PL(\mathcal M (\hat A^*_{n-1})) + \nabla_{a_{ij}} PL(\mathcal M(\hat A^*_{n-1})) \cdot \delta_{ij}
\end{equation}
where $\delta_{ij} = \pm 1$, representing that $\hat A$ adds/removes the link $\langle v_i, v_j \rangle$ given $\hat A^*_{n-1}$. Then, calculating $\nabla_{a_{ij}} PL(\mathcal M(\hat A^*_{n-1})), \forall a_{ij} \in \hat A^*_{n-1}$, denoted as $\nabla_{\hat A^*_{n-1}} PL(\mathcal M(\hat A^*_{n-1}))$, is the key issue of the POG method. With the chain rule, we infer that,
\begin{equation}
\nabla_{\hat A^*_{n-1}} PL(\mathcal M(\hat A^*_{n-1})) =  \nabla_{\hat A^*_{n-1}} \hat X^*_{n-1} \cdot \nabla_{\hat X^*_{n-1}} PL(\hat X^*_{n-1})
\end{equation}
where $\hat X^*_{n-1} = \mathcal M(\hat A^*_{n-1})$. While $\hat A^*_{n-1}$ and $\hat X^*_{n-1}$ are known in the $n$-th iteration, $\nabla_{\hat X^*_{n-1}} PL(\hat X^*_{n-1})$ can be easily computed by backpropagation. However, it is still non-trivial to efficiently compute $\nabla_{\hat A^*_{n-1}} \hat X^*_{n-1} = \nabla_{\hat A^*_{n-1}} \mathcal M (\hat A^*_{n-1})$ as the embedding $\mathcal M$ is often a complicated computation process. 

While designing an efficient and general computation method for an arbitrary embedding method is almost impossible, we notice that a stream of 
popular skip-gram embedding methods, such as DeepWalk \cite{perozzi2014deepwalk} and LINE \cite{tang2015line}, present very promising theoretical properties with embedding in closed forms \cite{qiu2018network}. This gives the opportunities to efficiently compute $\nabla_{\hat A^*_{n-1}} \mathcal M (\hat A^*_{n-1})$. Specifically, these methods are equivalent to the factorization of a certain matrix, 
\begin{equation}
    Z = XY^T
\end{equation}
where $X$ is equivalent to $\mathcal M(A)$, and the matrix $Z$ is derived from the network adjacency matrix $A$. For DeepWalk,
\begin{equation}\label{eq:deepwalk}
Z=\log (vol(A)(\dfrac{1}{T} \sum_{r=1}^{T} (D^{-1}A)^r)D^{-1}) - \log b,
\end{equation}
For LINE, 
\begin{equation}\label{eq:line}
Z=\log (vol(A)(D^{-1}AD^{-1})) - \log b,
\end{equation}
For more details of $Z$, readers can refer to \citet{qiu2018network}. In this vein, supposing that we have the matrix $\hat Z^*_{n-1}$ of $\hat A^*_{n-1}$, we can infer,
\begin{equation}
    \nabla_{\hat A^*_{n-1}} \hat X^*_{n-1} = \nabla_{\hat Z^*_{n-1}} \hat X^*_{n-1} \cdot \nabla_{\hat A^*_{n-1}} \hat Z^*_{n-1}
\end{equation}
The two components in the right side of the equation can both be efficiently computed with the first order KKT (Karuch-Kuhn-Tucker) condition \cite{sun2018data}, and thus $\nabla_{\hat A^*_{n-1}} \hat X^*_{n-1}$ is computed.

In a word, we develop a fast training-free method to estimate $PG(\mathcal M (\hat A)|\mathcal M (A)), \forall \hat A \in \mathbb A_n$ by $\nabla_{\hat A^*_{n-1}} PL(\mathcal M(\hat A^*_{n-1}))$. Specifically, the estimated privacy gain of $\hat A \in \mathbb A_n$ over $\hat A^*_{n-1}$ is,
\begin{equation}
 \hat{PG} (\mathcal M (\hat A)|\mathcal M (\hat A^*_{n-1})) = - \nabla_{a_{ij}} PL(\mathcal M(\hat A^*_{n-1})) \cdot \delta_{ij}, \forall \hat A \in \mathbb A_n
 \label{eq:new_privacy_gain}
\end{equation}
where $\hat A$ is different from $\hat A^*_{n-1}$ only at $a_{ij}$.

\subsubsection{Training-free Utility Loss Computation} Clearly, it is also extremely inefficient to compute utility loss by training the embedding for every $\hat A \in \mathbb A_{n}$ in each iteration of PPNE. 
Here, we further design an efficient training-free utility loss computation method for the skip-gram embedding methods including DeepWalk and LINE.

Specifically, for $n$-th iteration's $\hat A \in \mathbb A_{n}$, we decompose the utility loss as the sum of the embedding loss $\mathcal L(\mathcal M(\hat A); \hat A)$ and the perturbation loss $||\hat A - \hat A^*_{n-1}||$,
\begin{equation}
	UL(\mathcal M (\hat A)|\mathcal M (\hat A^*_{n-1})) = \mathcal L(\mathcal M (\hat A); \hat A) + ||\hat A - \hat A^*_{n-1}|| 
\end{equation}
For any $\hat A \in \mathbb A_n$, $||\hat A - \hat A_{n-1}^*||$ is the same (different by one link). Directly computing the embedding loss $\mathcal L(\mathcal M(\hat A); \hat A)$ for all $\hat A \in \mathbb{A}_n$ still needs a lot of time. 

Fortunately, given a network $A$, the previous work \cite{bojchevski2019adversarial} deduces an analytic equation that can efficiently approximate the DeepWalk embedding loss of another network $\hat A$, if $\hat A$ is different from $A$ in only one link,
\begin{equation}
    \hat{\mathcal{L}}(\mathcal M(\hat A); \hat A)=\dfrac{vol(A) + 2\Delta w_{ij}}{T \cdot b} [\sum\nolimits_{p=K+1}^{|V|}\tilde{\sigma}_{p}^{2}]^{1/2}
    \label{eq:deepwalk_loss_est}
\end{equation}
where $\Delta w_{ij}=\pm 1$ represents that $\hat A$ adds or removes $\langle v_i, v_j \rangle$ from $A$; $\tilde{\sigma}_{p}=\frac{1}{d_{min}} \cdot |\sum_{r=1}^{T} (\tilde{\lambda}_{\pi (p)})^r|$, $\tilde{\lambda}$ is the approximation of the generalized eigenvalue $\lambda$ of $\hat A$ by solving $\hat A u=\lambda D u$; $\pi$ is a permutation ensuring that the $\tilde{\sigma_{p}}$ are sorted decreasingly, and $d_{min}$ is the smallest degree in $\hat A$. For details, readers can refer to the work \cite{bojchevski2019adversarial}.

We note again that $\hat A$ differs from $\hat A_{n-1}^*$ with only one element. Therefore, for DeepWalk embedding, we propose to leverage the above idea from \cite{bojchevski2019adversarial} to quickly approximate $\mathcal L(\mathcal M (\hat A); \hat A)$ without explicitly training $\mathcal M(\hat A)$.

It is worth noting that, the basic technique used in \cite{bojchevski2019adversarial} is again converting the DeepWalk embedding training to the matrix factorization \cite{qiu2018network}. Meanwhile, from Eq.~\ref{eq:deepwalk} and \ref{eq:line}, we can know that LINE is a special case of DeepWalk by setting $T=1$ from the matrix factorization's perspective. Hence, we can obtain the approximate embedding loss of LINE by setting $T=1$ in Eq.~\ref{eq:deepwalk_loss_est},
\begin{equation}
    \hat{\mathcal{L}}(\mathcal M(\hat A); \hat A)=\dfrac{vol(A) + 2\Delta w_{ij}}{b} [\sum\nolimits_{p=K+1}^{|V|}\tilde{\sigma}_{p}^{2}]^{1/2}
    \label{eq:line_loss_est}
\end{equation}
By setting $A$ to $\hat A_{n-1}^*$ in Eq.~\ref{eq:deepwalk_loss_est}/\ref{eq:line_loss_est}, we can efficiently approximate the DeepWalk/LINE embedding loss for all $\hat A \in \mathbb A_n$ without explicitly training $\mathcal M(\hat A)$. Then, the estimated utility loss of $\hat A$ over $\hat A^*_{n-1}$  is,
\begin{equation}
    \hat{UL} (\mathcal M (\hat A)|\mathcal M(\hat A_{n-1}^*) = \hat{\mathcal{L}}(\mathcal M(\hat A); \hat A),\ \forall \hat A \in \mathbb A_n
    \label{eq:new_utility_loss}
\end{equation}

\subsubsection{Fast Optimal Perturbation Selection}

By combining the efficient estimation of the privacy gain (Eq.~\ref{eq:new_privacy_gain}) and utility loss (Eq.~\ref{eq:new_utility_loss}), we can select the optimal $\hat A \in \mathbb A_n$ in the $n$-th iteration by,
\begin{equation}
	\hat A^*_n = \argmax_{\hat A \in \mathbb{A}_n} \frac{\hat {PG}(\mathcal M (\hat A)|\mathcal M(\hat A_{n-1}^*))}{[\hat{UL}(\mathcal M (\hat A)|\mathcal M(\hat A_{n-1}^*))]^k}
	\label{eq:pu_tradeoff_new}
\end{equation}
With Eq.~\ref{eq:pu_tradeoff_new}, we do not need to train embedding $\mathcal M(\hat A)$ for each candidate perturbation $\hat A \in \mathbb{A}_n$. Instead, only the embedding of $\mathcal M(\hat A_{n-1}^*)$ is needed for computing Eq.~\ref{eq:pu_tradeoff_new}, which significantly boosts the scalability of PPNE.

\subsubsection{Perturbation Sampling and Batching}

The aforementioned fast approximations on privacy gain and utility loss have eliminated a large number of embedding trainings. However, if the network is huge, e.g., up to millions of nodes, the computation complexity is still high as the number of candidate network perturbations in each iteration is $O(|V|^2)$. We propose to further reduce the computation overhead by the \textit{perturbation sampling and batching} strategy.

\textbf{Perturbation Sampling}. In each iteration, we randomly sample a set of perturbations from all the possible ones, and the utility loss  and privacy gain computation processes are carried out only on the sampled perturbations. 
In brief, PPNE with a larger sampling size would achieve better privacy-utility tradeoff, but with higher computation overhead. 


\textbf{Perturbation Batching}. In each iteration, instead of selecting the best (top-$1$) perturbation, we select a batch of advantageous (the top-$f$) perturbations based on the privacy gain and utility loss to further accelerate PPNE. A larger $f$ could significantly reduce the number of iterations to achieve a expected privacy gain.

Finally, we note that the computation of PPNE is easy to be parallelized as the computations of the privacy gain and utility loss for different network perturbations are independent of each other. 
\section{Evaluation}
\subsection{Experimental Setups}
\subsubsection{Datasets}
We use four well-known datasets to evaluate PPNE, including Cora~\cite{sen2008collective}, Citeseer~\cite{sen2008collective}, PubMed~\cite{dai2018adversarial}, and Flickr~\cite{tang2015line}. In each dataset, we construct a set of target node pairs (\emph{i.e.,} $\mathcal P$) by randomly sampling a portion of links as private links for positive node pairs (\emph{i.e.,} $\mathcal P_{pos}$) and the same number of unlinked node pairs as the negative node pairs (\emph{i.e.,} $\mathcal P_{neg}$). We respectively use DeepWalk~\cite{perozzi2014deepwalk} and LINE~\cite{tang2015line} to generate network embedding. The dataset characteristics and the experimental parameters are summarized in Table ~\ref{tbl:datasetstats}. All experiments are run on a server with 24-core Intel CPU, 256 GB RAM and Tesla V100S GPU. The operation system is Ubuntu 18.04.5 LTS. We implement PPNE with Python 3.7.

\begin{table}[t]
    \centering
	\scriptsize
	\renewcommand*{\arraystretch}{1.1}
	\caption{Dataset statistics and default settings of parameters}
	\vspace{-1.5em}
    \begin{tabular}{c|ccc|cccc}
    \toprule
    Dataset & Nodes & Links & Labels & |$\mathcal P_{pos}|$/$|\mathcal P_{neg}|$ & $k$ & $s$ & $f$  \\ \midrule
    Cora & 2,708 & 5,429 & 6 & 10\% & 1 & 10,000 & 1 \\
    Citeseer & 3,327 & 4,732 & 7 & 10\% & 1 & 10,000 & 1 \\
    PubMed & 19,717 & 44,338 & 3 & 2\% & 1 & 50,000 & 100 \\
    Flickr & 1,715,255 & 15,555,042 & 5 & 3\% & 1 & 500,000 & 10,000 \\
    \bottomrule
    \multicolumn{8}{l}{$k$: the relative utility importance over privacy~in Eq.~\ref{eq:pu_tradeoff_new}.}\\
    \multicolumn{8}{l}{$s$: the perturbation sampling size regarding PPNE acceleration.}\\
    \multicolumn{8}{l}{$f$: the perturbation batch size regarding PPNE acceleration.}
	\end{tabular}
	\label{tbl:datasetstats}
	\vspace{-2.5em}
\end{table}

\subsubsection{Link inference attacks} We specify two attack scenarios with respect to attackers' auxiliary knowledge, i.e., attackers with/without auxiliary knowledge. While we leave the experiments regarding the attackers with some auxiliary knowledge in Appendix due to the page limitation, we assume attackers only acquire the published network embedding but no extra information in the main text. In particular, the attackers compute a private link probability for each target node pair based on the cosine similarities of their embedding by a sigmoid function. 
The performance of link inference attacks is often measured by \textit{Average Precision Score (AP score)}, which  calculates the weighted mean of precisions at different thresholds of the precision-recall curve. As we intend to protect private links from being correctly inferred, we use $1-AP\ score$ to quantify the privacy protection effect. Larger $1- AP\ score$ indicates better privacy protection and less privacy leakage.


\subsubsection{Data utility applications}
The released network embedding should retain utility as much as possible for downstream tasks. In this work, we consider two popular applications for network embedding utility evaluation.

\textbf{Node classification.} It aims to correctly classify nodes based on a proportion of labeled nodes by a supervised learning paradigm. Specifically, we randomly choose 70\% of nodes as the training set and take the remaining 30\% nodes' labels for testing. Following the most prior studies \cite{grover2016node2vec,tang2015line,perozzi2014deepwalk}, we use logistic regression as the classification model and $F1\ score$ to assess the performance of node classification. Concerning the $F1\ score$ of a perfect node classification could equal $1$, we take $1-F1\ score$ to measure the utility loss of the perturbed embedding.

\textbf{Node clustering.} We also run an unsupervised learning to cluster nodes based on their embedding similarity. It is often hard to evaluate the performance of node clustering directly due to the lack of its ground truth. Fortunately, our goal is to verify that the clustering performance would not decrease much when the perturbed network embedding substitutes the original one. Therefore, we assess the difference of clustering result based on the perturbed network embedding to that using the original one by \textit{Normalized Mutual Information (NMI)} between the clustering results. NMI equals $1$ if the clustering results based on the perturbed embedding is equivalent to those of the original embedding. Thus, we use $1-NMI$ to assess the utility loss caused by the network perturbations.

\subsubsection{Baselines}
We compare PPNE to extensive baselines as:

$\bullet$ \textit{Random}~\cite{bojchevski2019adversarial}. This method generates a perturbed network $A'$ by randomly adding or deleting links given network $A$.

$\bullet$ \textit{Degree}~\cite{chang2020restricted}. This approach adds or removes a link between two nodes based on the sum of their degree. The link between two nodes is more likely deleted if their degree sum is lower, and a link is more likely added between two nodes without a public link if their degree sum is larger.

$\bullet$ \textit{Betweenness}~\cite{ma2020towards}. This approach ranks the overall links in a network by their betweenness centrality and delete links in order.

$\bullet$ \textit{DICE}~\cite{zugner2019adversarial}. This approach only deletes links that connect to the nodes with some private links, and adds links between the nodes with no private links.

$\bullet$ \textit{DP}~\cite{sala2011sharing}. It first adds noise to the degrees of the original network and then generates a differentially private network based on the perturbed node degrees. We use the differentially private network as the perturbed network to generate embedding for publishing.

$\bullet$
\textit{DPNE}~\cite{xu2018dpne} Instead of perturbing the original network, DPNE takes network embedding as a matrix factorization problem, and adds random noise to the objective function to generate differentially private network embedding.

Note that \textit{Random}, \textit{Degree}, \textit{Betweenness}, \textit{DICE} and \textit{DP} are network perturbation methods as PPNE. We use these methods to perturb the original network and then take the same procedure to produce the perturbed embedding accordingly for comparisons. The embedding generated by \textit{DPNE} is directly used to compare with.

\begin{figure*}[t]
	\centering
	\begin{subfigure}[b]{0.23\textwidth}
		\includegraphics[width=\textwidth]{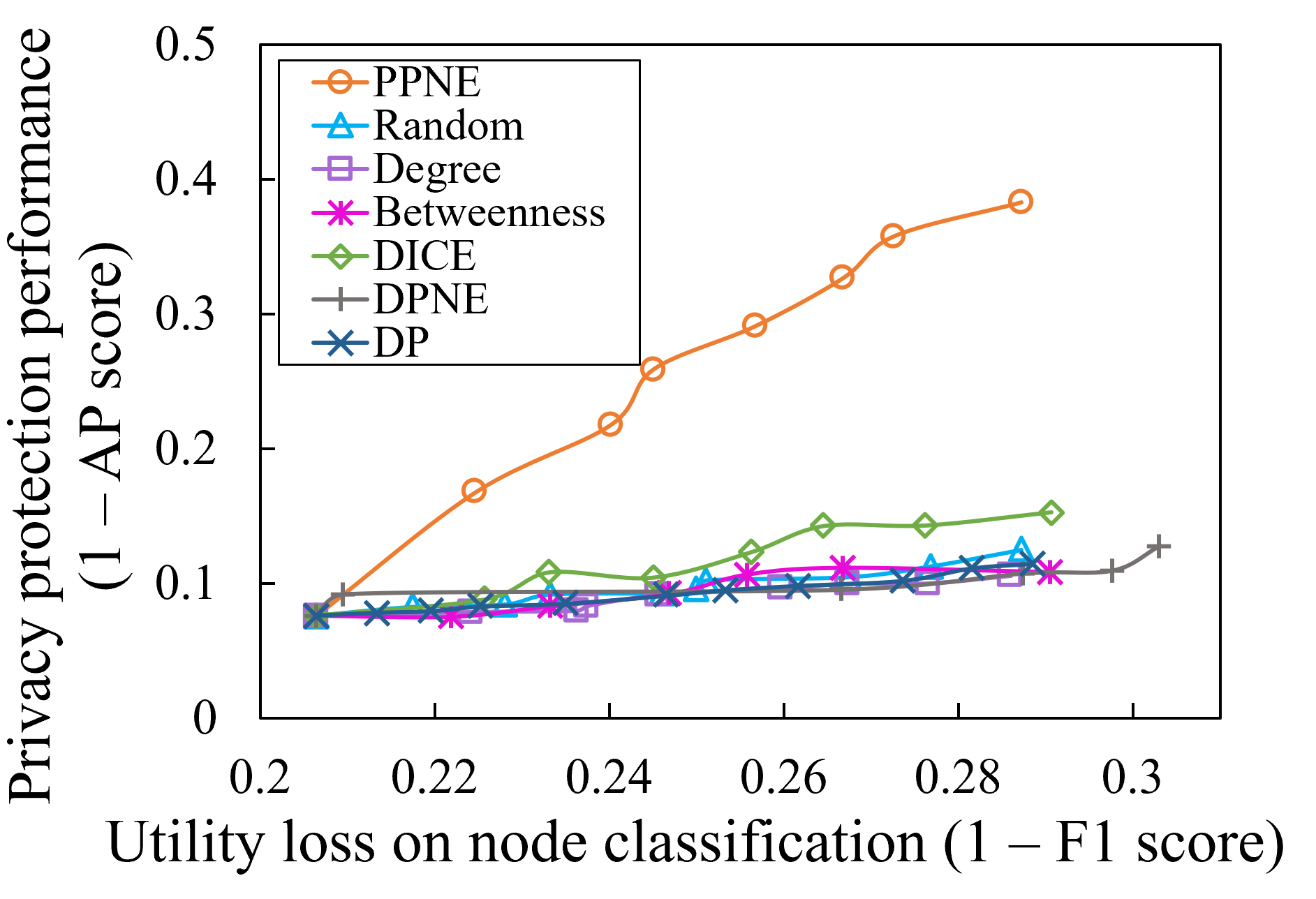}
		\caption{Classification (DeepWalk)}
		\label{fig:DWSINC}
	\end{subfigure}
	\quad
	\begin{subfigure}[b]{0.23\textwidth}
		\includegraphics[width=\textwidth]{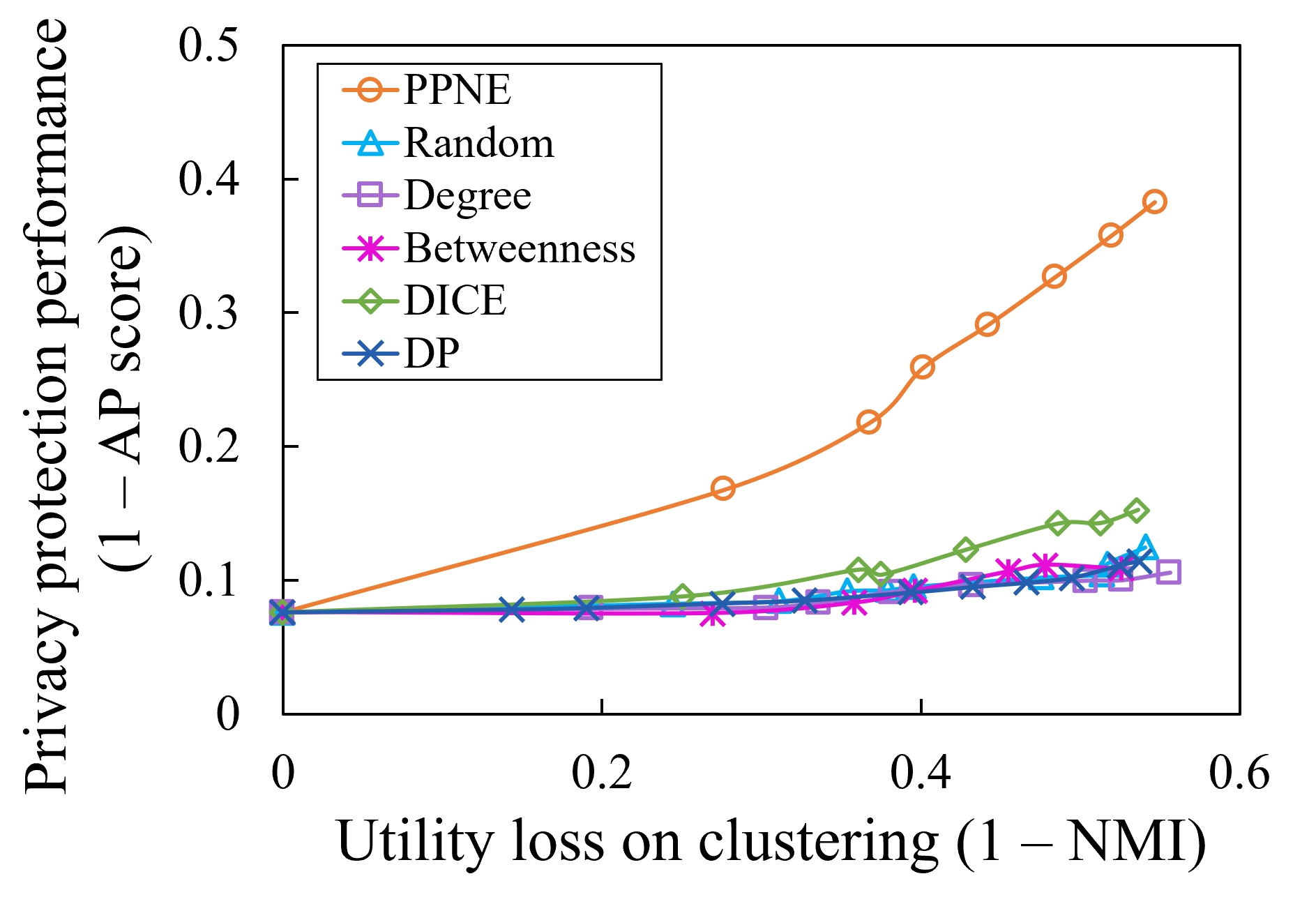}
		\caption{Clustering (DeepWalk)}
		\label{fig:DWSICL}
	\end{subfigure}
	\quad
	\begin{subfigure}[b]{0.23\textwidth}
		\includegraphics[width=\textwidth]{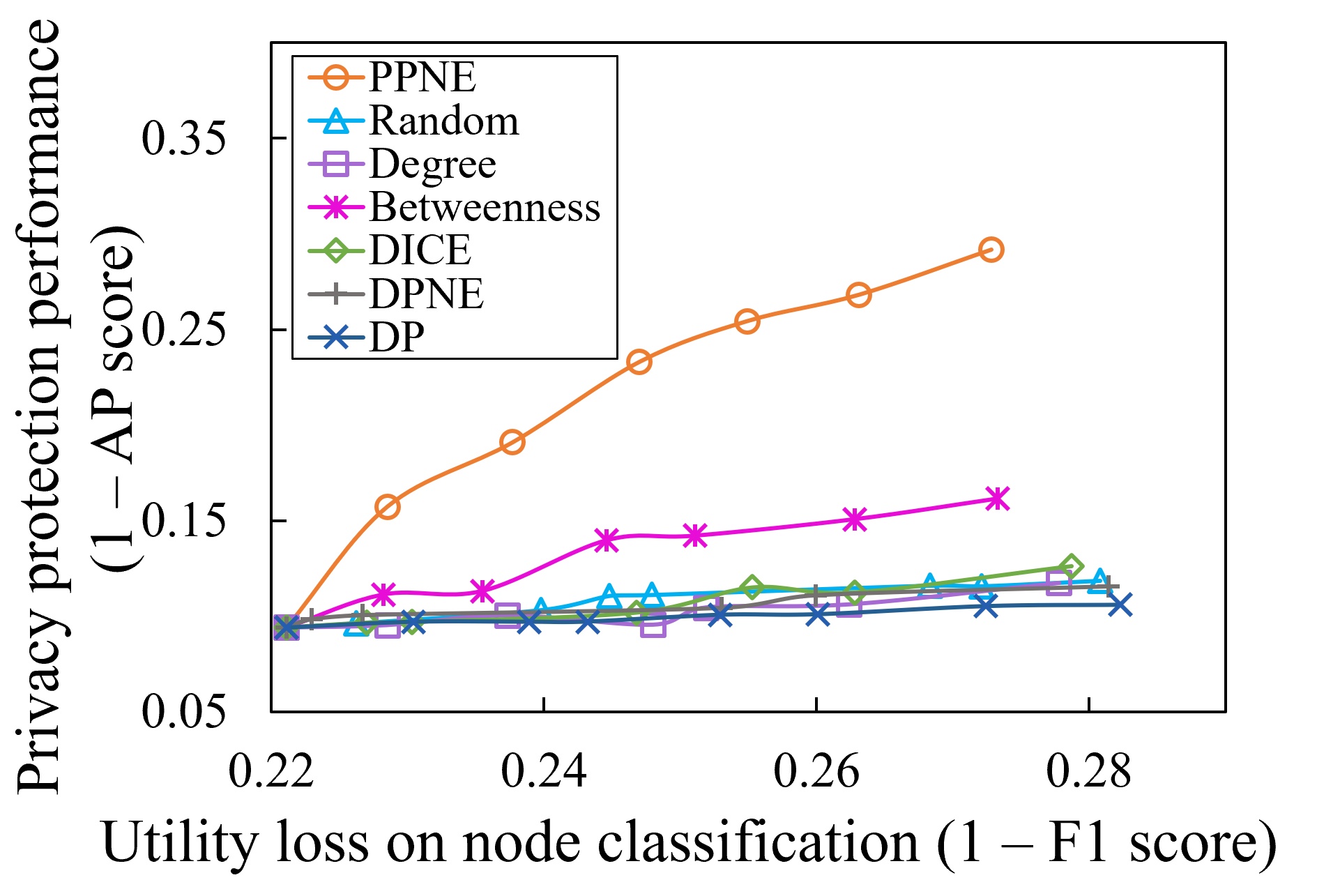}
		\caption{Classification (LINE)}
		\label{fig:LINESINC}
	\end{subfigure}
	\quad
	\begin{subfigure}[b]{0.23\textwidth}
		\includegraphics[width=\textwidth]{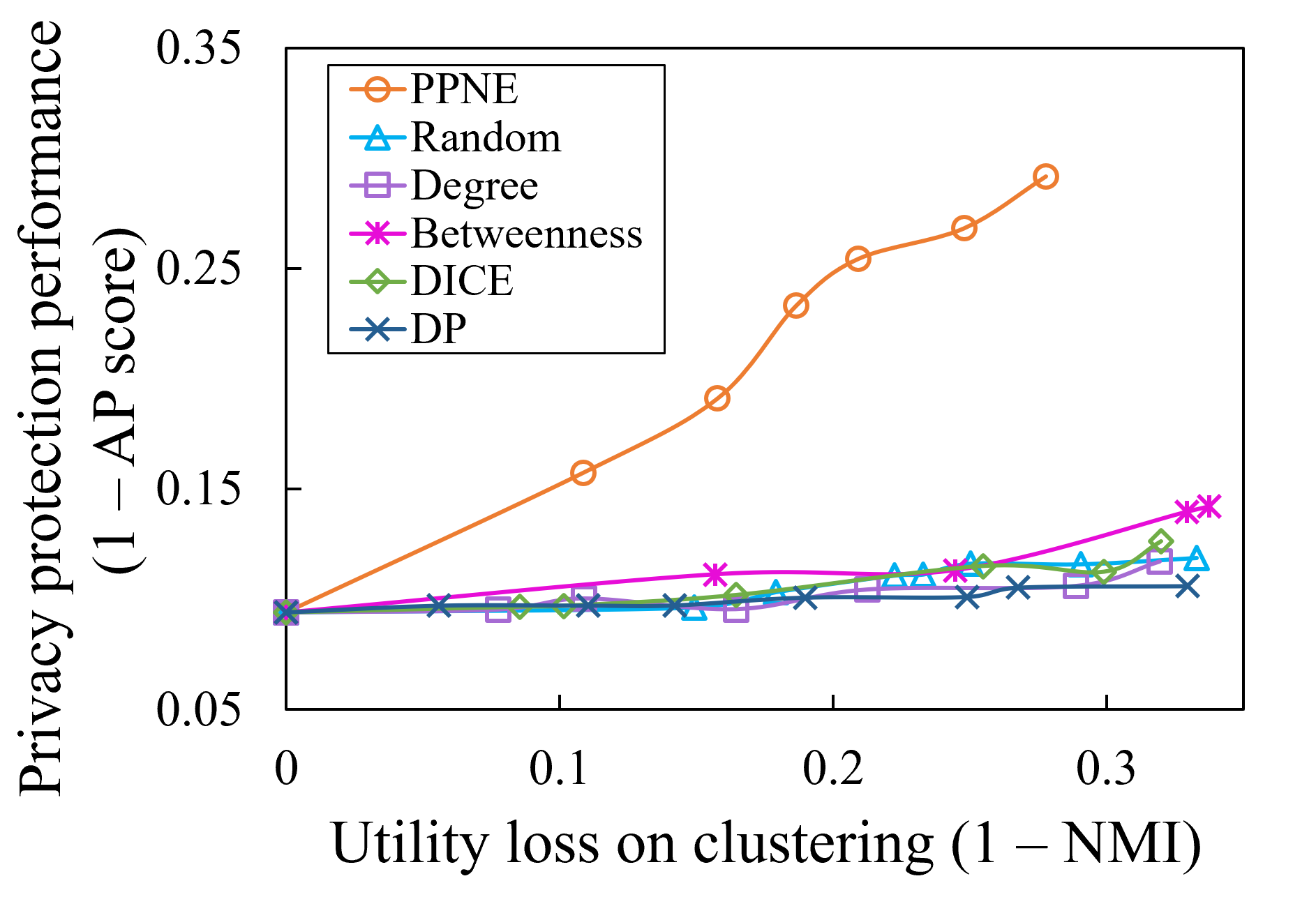}
		\caption{Clustering (LINE)}
		\label{fig:LINESICL}
	\end{subfigure}
	\vspace{-1em}
	\caption{Tradeoff on Cora: protection on private links and node classification/clustering performance }\label{fig:TradeOffDW}
	\vspace{-1em}
\end{figure*}


\begin{figure*}[t]
    \begin{subfigure}[b]{0.23\textwidth}
        \includegraphics[width=\textwidth]{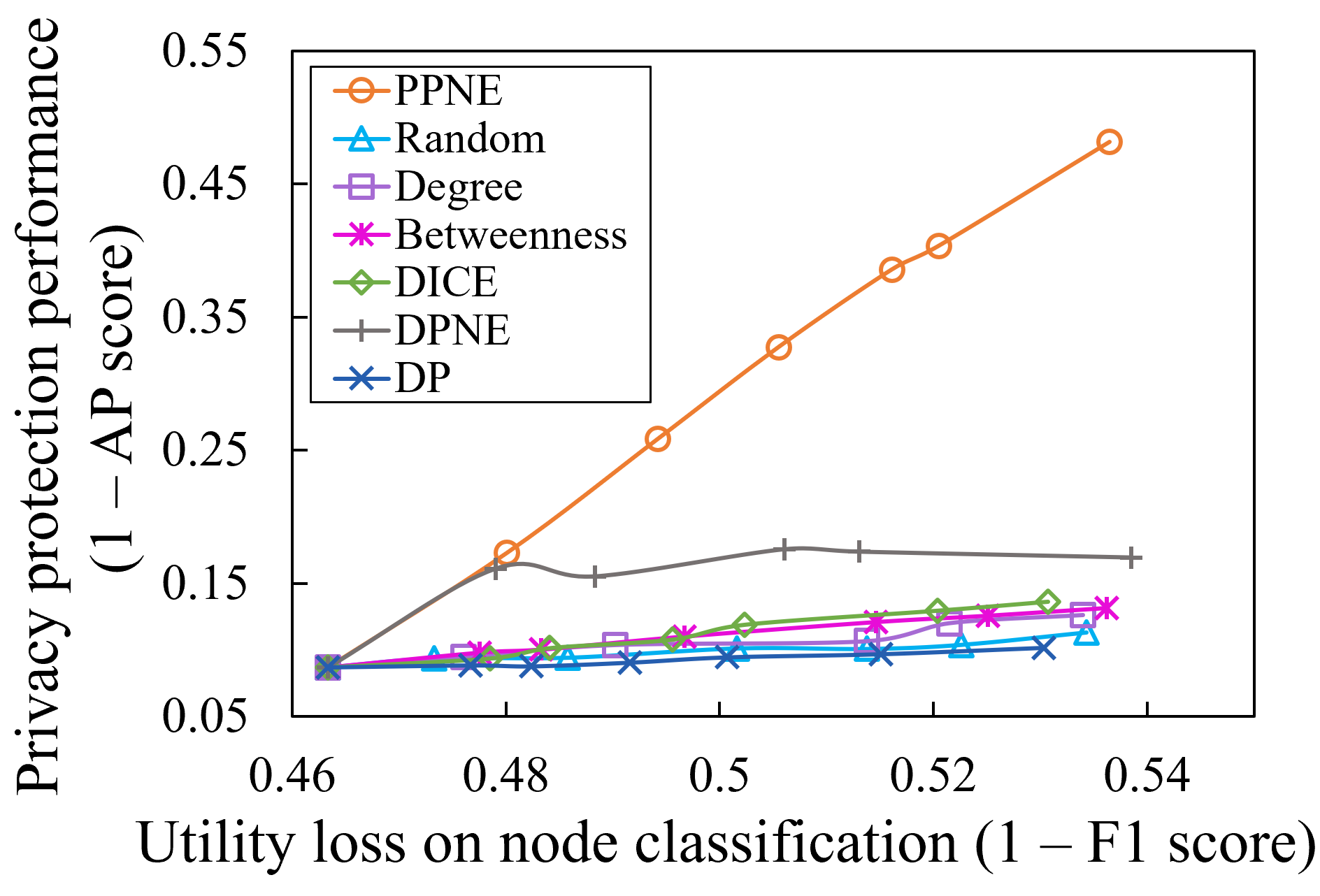}
        \caption{Citeseer}
    \end{subfigure}
    \quad \quad \quad
    \begin{subfigure}[b]{0.23\textwidth}
        \includegraphics[width=\textwidth]{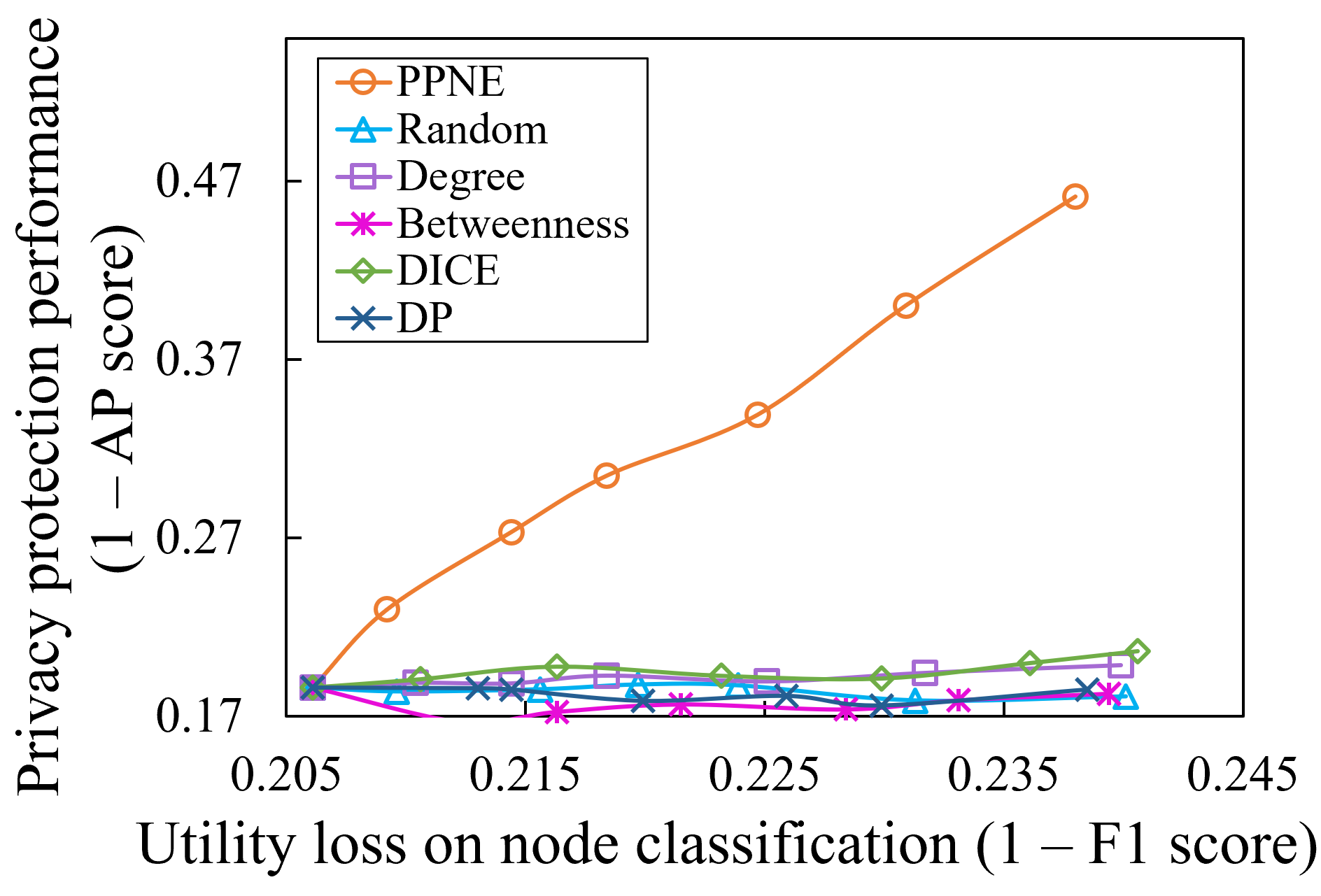}
        \caption{PubMed}
    \end{subfigure}
    \quad \quad \quad
    \begin{subfigure}[b]{0.23\textwidth}
        \includegraphics[width=\textwidth]{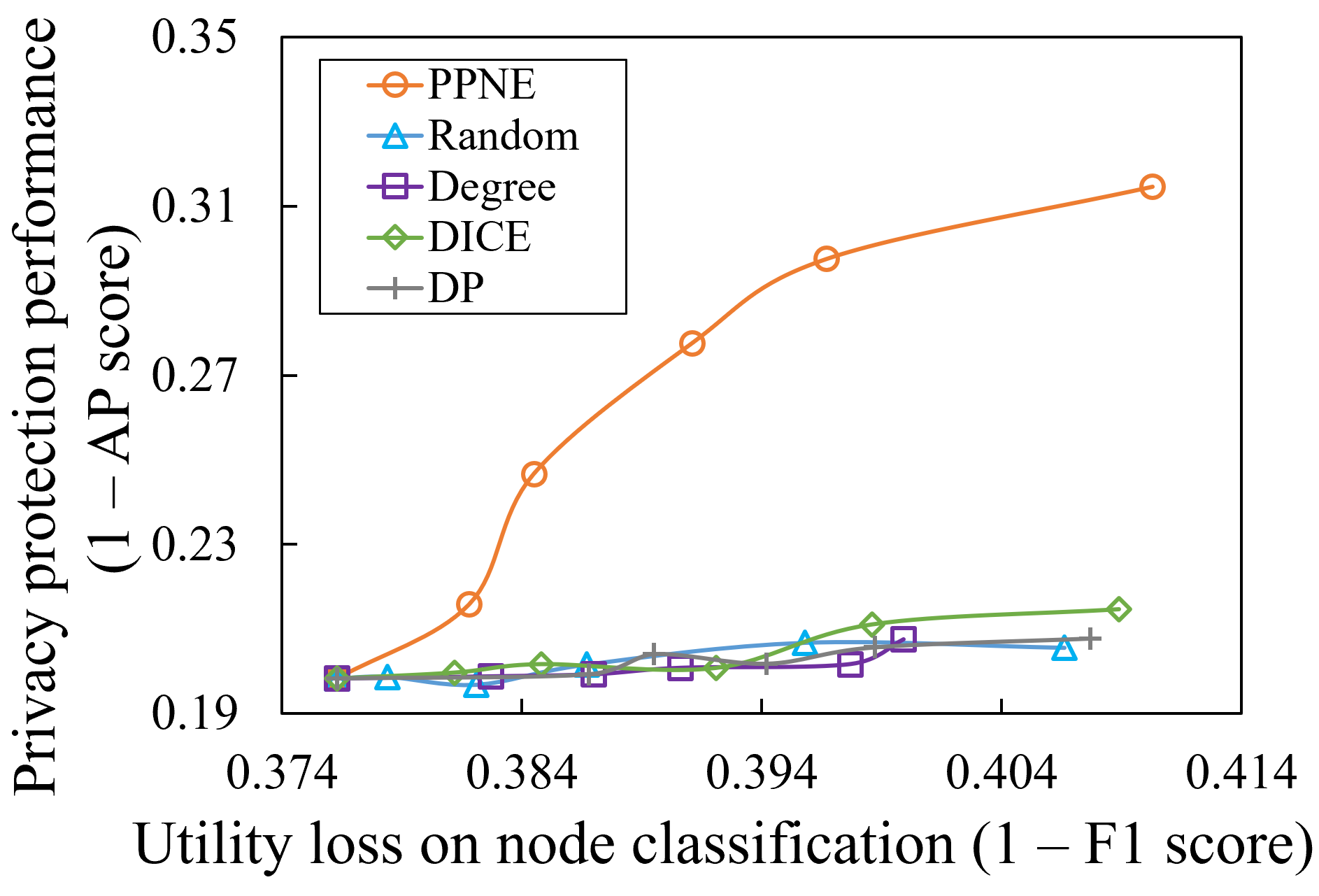}
        \caption{Flickr}
    \end{subfigure}
    \vspace{-1em}
    \caption{privacy protection on private links and node classification performance on other datasets (LINE)}
    \label{fig:other_data}
    \vspace{-1em}
\end{figure*}

\subsection{Experimental Results}

\subsubsection{Comparison with baselines}

Fig.~\ref{fig:TradeOffDW} compares the baselines with PPNE on Cora, in which the increase of value on y-axis represents the better privacy protection while a larger value on x-axis indicates a higher utility loss. 
From the experimental results, we find that PPNE and the baselines show the same trend that privacy protection effect is improved with the decline in utility level, which verifies that there is a tradeoff between the privacy protection and data utility. More importantly, we observe that PPNE always outperforms the baselines under various experimental settings (\emph{i.e.,} different applications and embedding methods). In other words, PPNE can offer stronger privacy protection by sacrificing the same amount of utility and thus achieves the best privacy-utility tradeoff. In particular, in Fig. \ref{fig:DWSINC}, when the utility level on node classification task decreases to around 0.24, the privacy protection performance brought by PPNE can be improved by 0.2 compared to initial point, while the privacy protection of other baselines are only increased by around 0.02. In Fig. \ref{fig:DWSICL}, we observe that when the (1-NMI) value is 0.4 on clustering task, the privacy protection provided by PPNE is higher than other baselines by more than 100\%.\footnote{DPNE's clustering performance is much worse than the other methods under the same privacy protection level, and thus not plotted in the clustering result figures.}

Fig.~\ref{fig:other_data} illustrates the results on other datasets with the node classification task and LINE embedding. PPNE consistently outperforms baselines significantly. Note that some baselines (e.g., DPNE) are not shown in the large networks (e.g., Flickr), as their running efficiency is too low to obtain results in these large networks. This also validates the superior scalability of PPNE over state-of-the-art privacy-preserving embedding methods like DPNE~\cite{xu2018dpne}.

\begin{figure}[t]
	\centering
	\begin{subfigure}[b]{0.23\textwidth}
		\includegraphics[width=\textwidth]{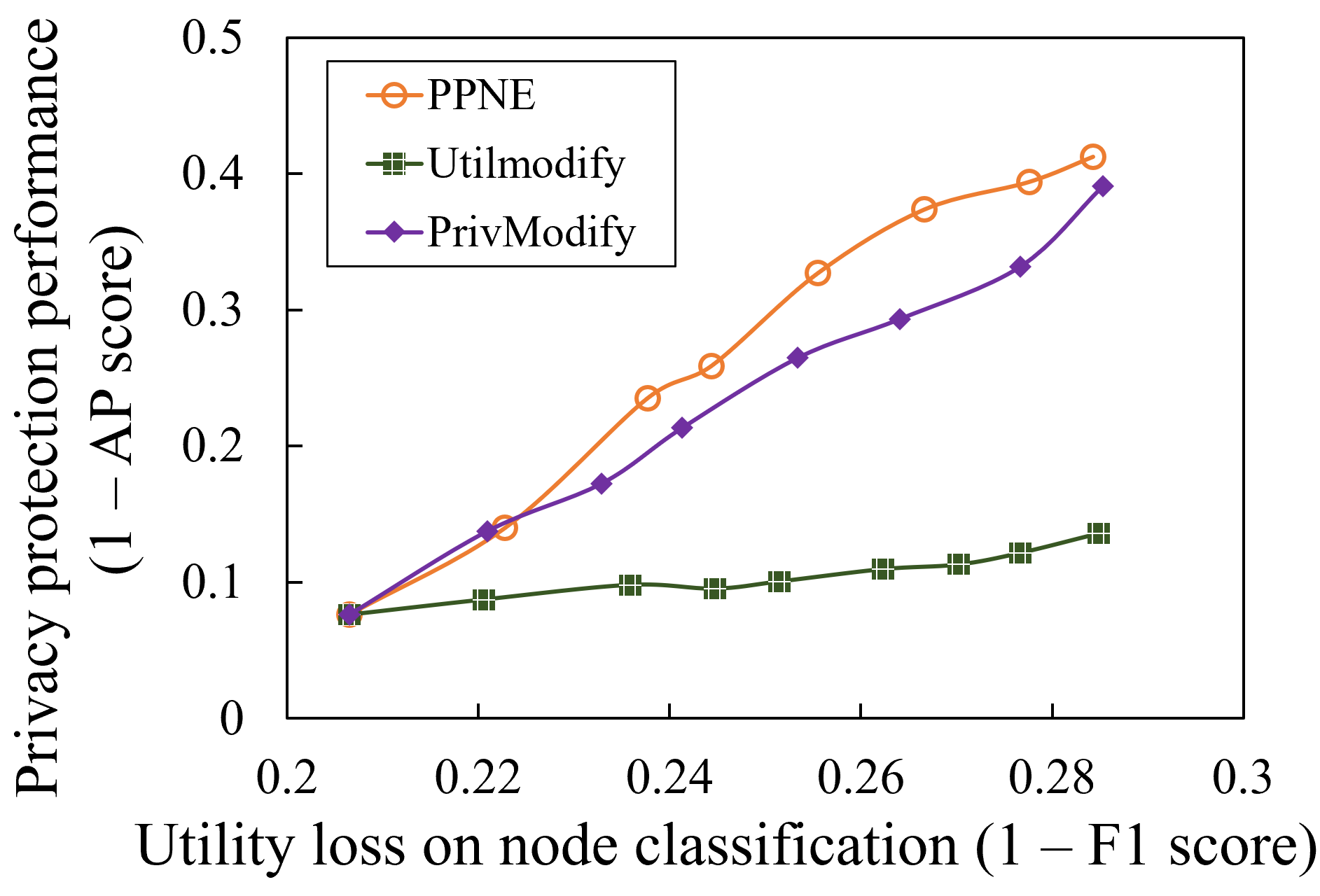}
		\caption{Node Classification}
		\label{fig:variantSINC}
	\end{subfigure}
	\begin{subfigure}[b]{0.23\textwidth}
		\includegraphics[width=\textwidth]{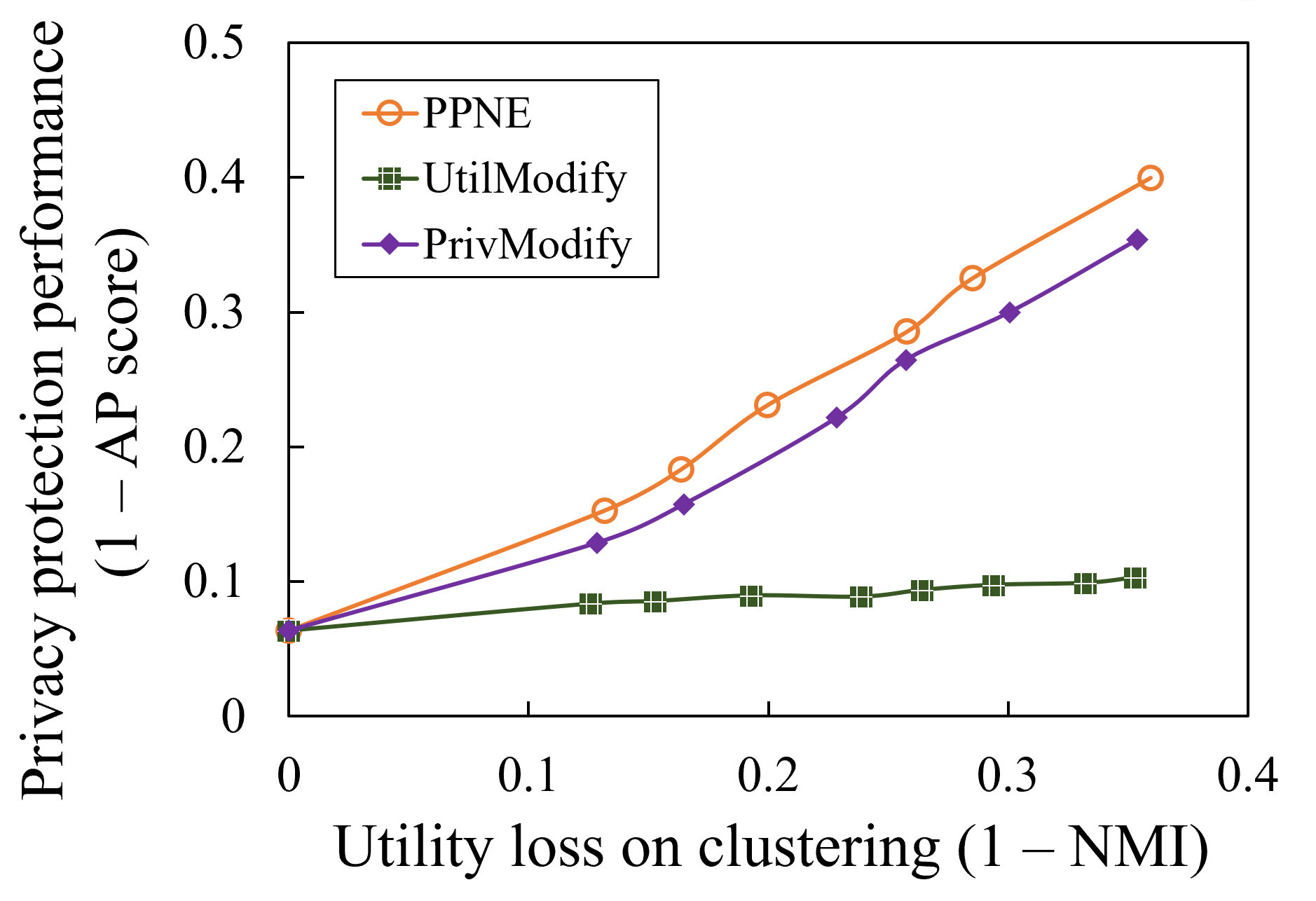}
		\caption{Node Clustering}
		\label{fig:variantSICL}
	\end{subfigure}
	\vspace{-1em}
	\caption{Comparison with PPNE variants (DeepWalk, Cora)}\label{fig:variants}
	\vspace{-1em}
\end{figure}


\subsubsection{Comparison with PPNE variants}

In addition to the baselines, we implement two variants of PPNE:

$\bullet$ \textit{UtilModify}: We rank the candidate perturbations by utility loss, and conduct the perturbation with the smallest utility loss.
	
$\bullet$ \textit{PrivModify}: This method only consider privacy gain for candidate perturbations selection.

Fig. \ref{fig:variants} shows that PPNE achieves better privacy-utility tradeoff than two variants. UtilModify performs really bad, indicating that only considering utility cannot ensure the improvement of privacy protection. 
In a word, by simultaneously considering privacy and utility, PPNE can trade more privacy with less utility loss.

\subsubsection{Parameter sensitivity}
We examine the impacts of key three parameters of PPNE. In particular, we vary the parameter of \textit{relative utility importance over privacy} $k$ in Eq.~\ref{eq:pu_tradeoff_new} on Cora; besides, we run experiments on PubMed with different \textit{perturbation sampling size} and \textit{perturbation batch size} $f$, as PubMed is a larger dataset. We only display the results based on DeepWalk.

\begin{figure}[t]
	\centering
	\begin{subfigure}[b]{0.23\textwidth}
		\includegraphics[width=\textwidth]{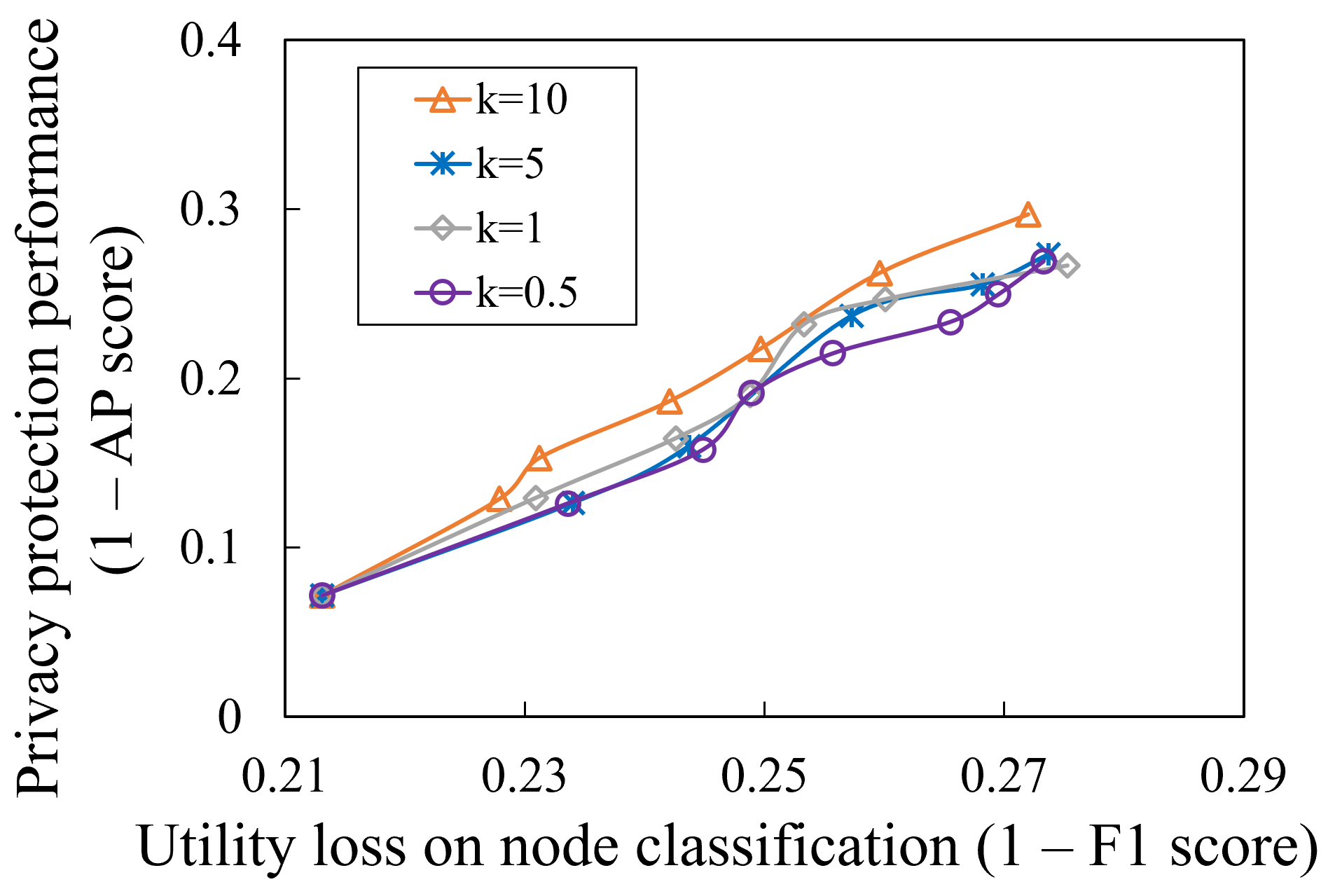}
		\caption{Node Classification}
		\label{fig:paramkSINC}
	\end{subfigure}
	~ 
	\begin{subfigure}[b]{0.23\textwidth}
		\includegraphics[width=\textwidth]{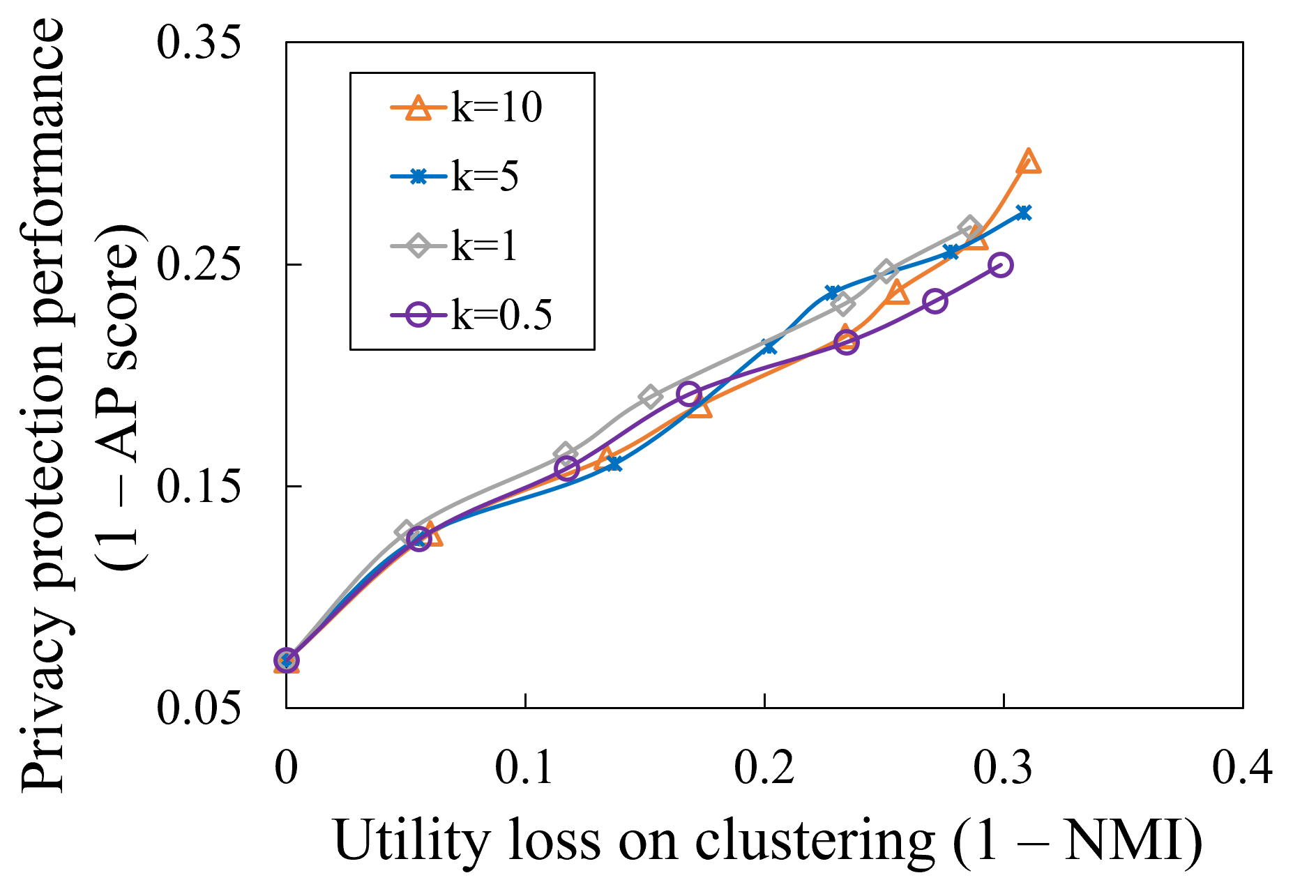}
		\caption{Node Clustering}
		\label{fig:paramkSICL}
	\end{subfigure}
	\vspace{-1em}
	\caption{Impacts of $k$ (DeepWalk)}\label{fig:KParam}
\end{figure}

\textbf{Relative utility importance over privacy} ($k$).
We set different values for $k$ to investigate how it will impact PPNE. From the results in Fig. \ref{fig:paramkSINC}, we observe that PPNE with $k=10$ obtains slightly better privacy-utility tradeoff performance compared to the other settings on node classification task. Meanwhile, the results in Fig.~\ref{fig:paramkSICL} show that $k$ does not show noticeable effects on the clustering task. Integrating the results with those of two PPNE variants, we conclude that the performance of PPNE may be improved by adjusting $k$ meticulously.

\begin{figure*}[t]
	\centering
	\begin{minipage}{.49\textwidth}
	\begin{subfigure}[b]{0.49\textwidth}
		\includegraphics[width=\textwidth]{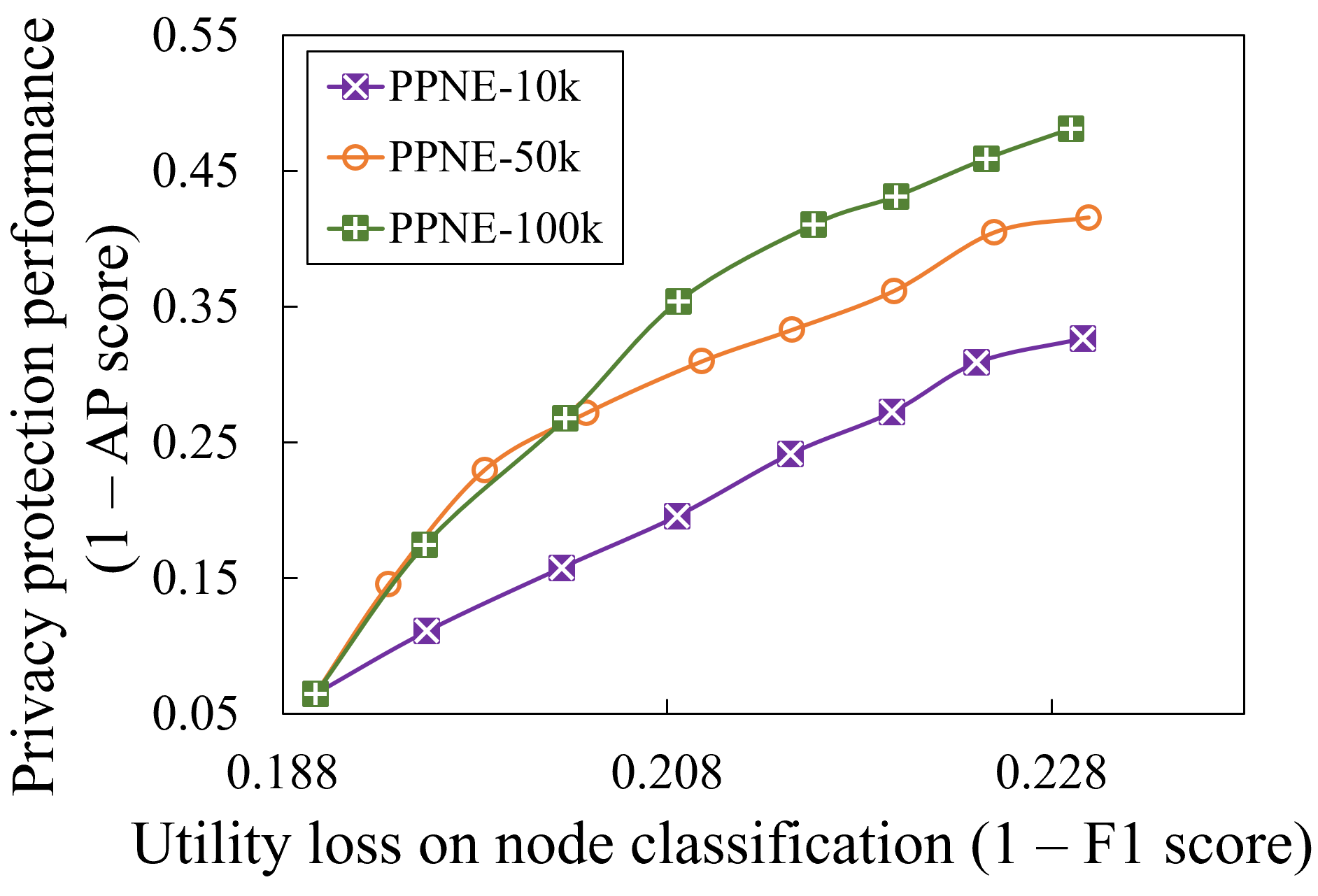}
		\caption{Node Classification}
		\label{fig:batchSINC}
	\end{subfigure}
	~ 
	\begin{subfigure}[b]{0.49\textwidth}
		\includegraphics[width=\textwidth]{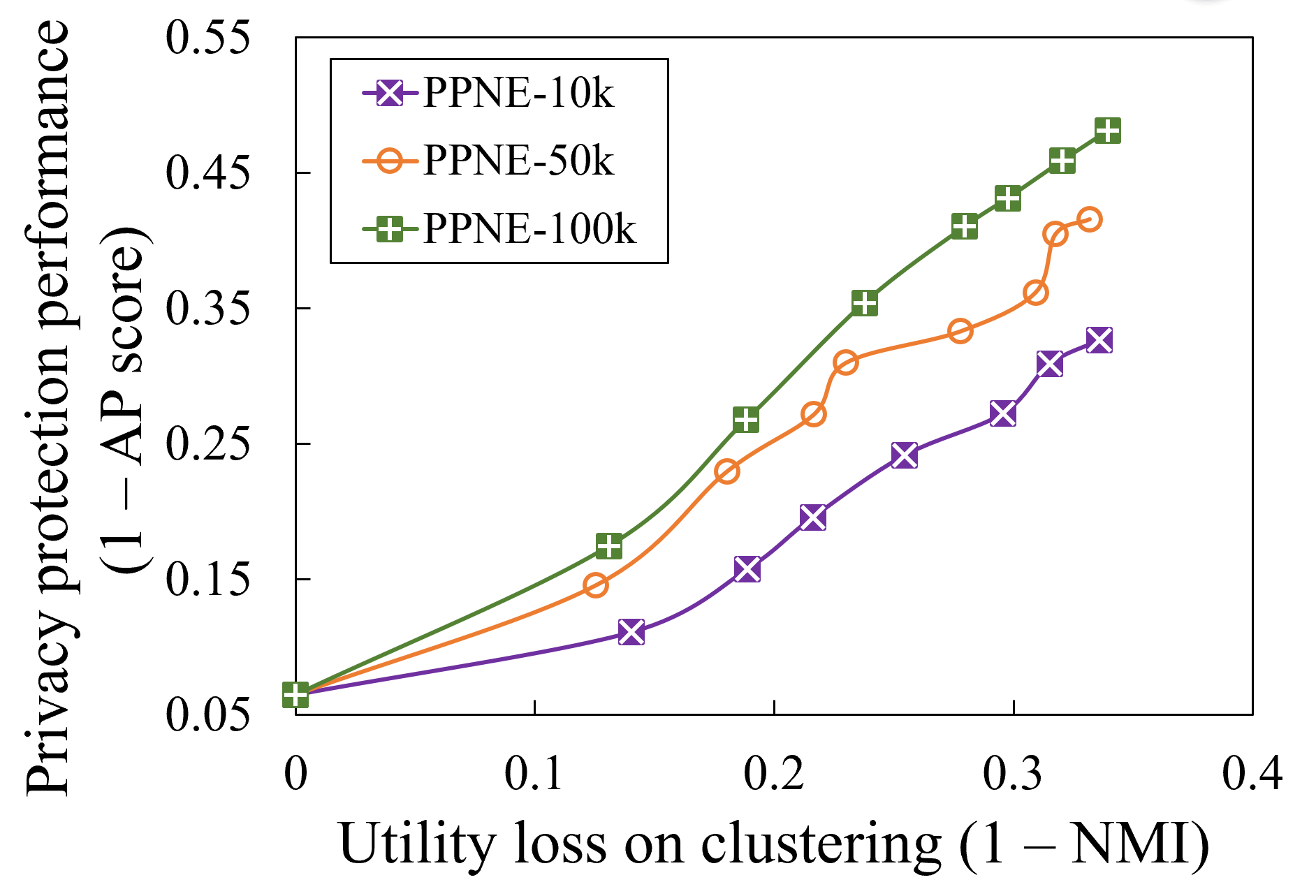}
		\caption{Node Clustering}
		\label{fig:batchSICL}
	\end{subfigure}
	\vspace{-1em}
	\caption{Impacts of $s$ (DeepWalk)}\label{fig:BatchParam}
	\end{minipage}
	\begin{minipage}{.49\textwidth}
	\begin{subfigure}[b]{0.49\textwidth}
		\includegraphics[width=\textwidth]{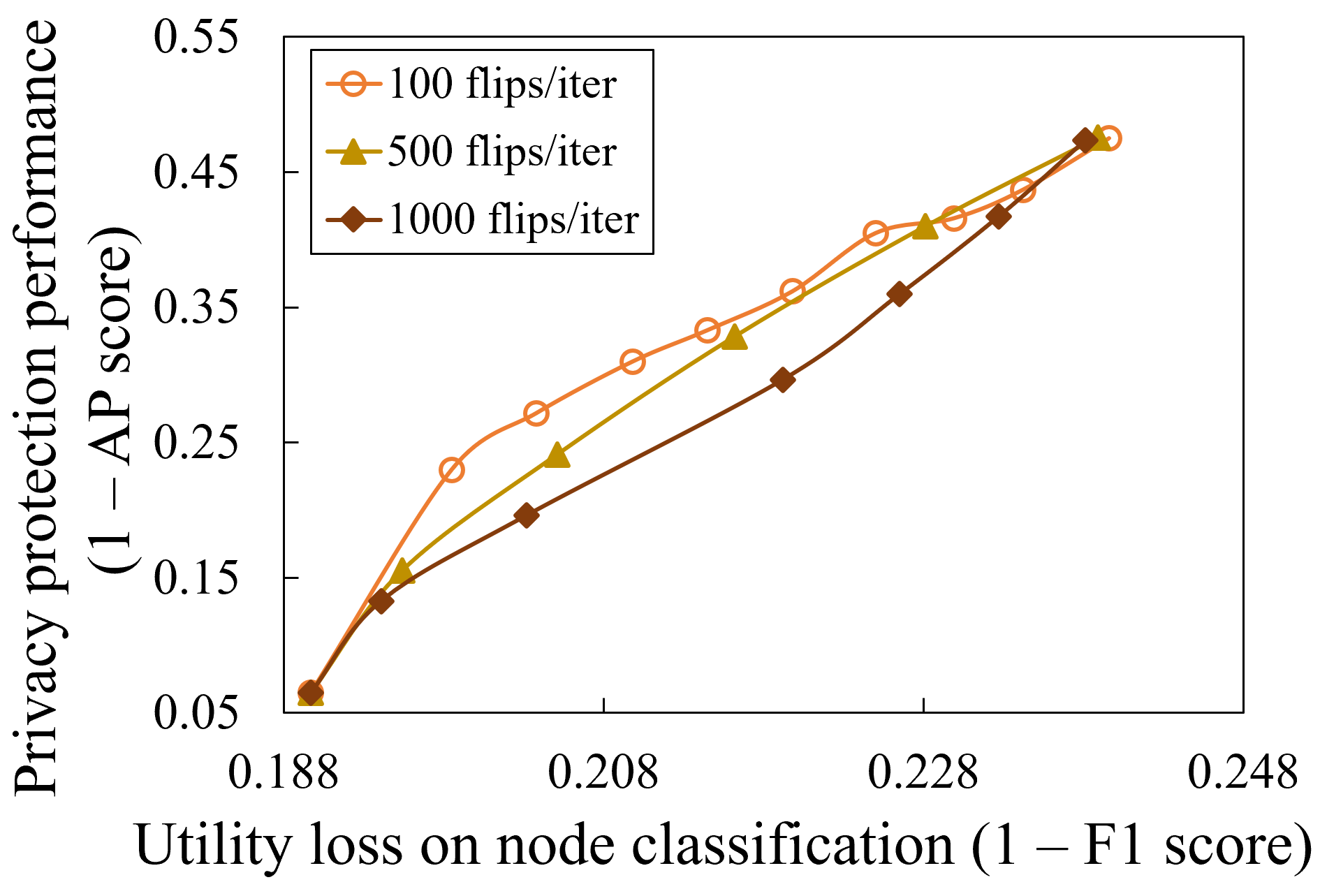}
		\caption{Node Classification}
		\label{fig:flipnumSINC}
	\end{subfigure}
	~
	\begin{subfigure}[b]{0.49\textwidth}
		\includegraphics[width=\textwidth]{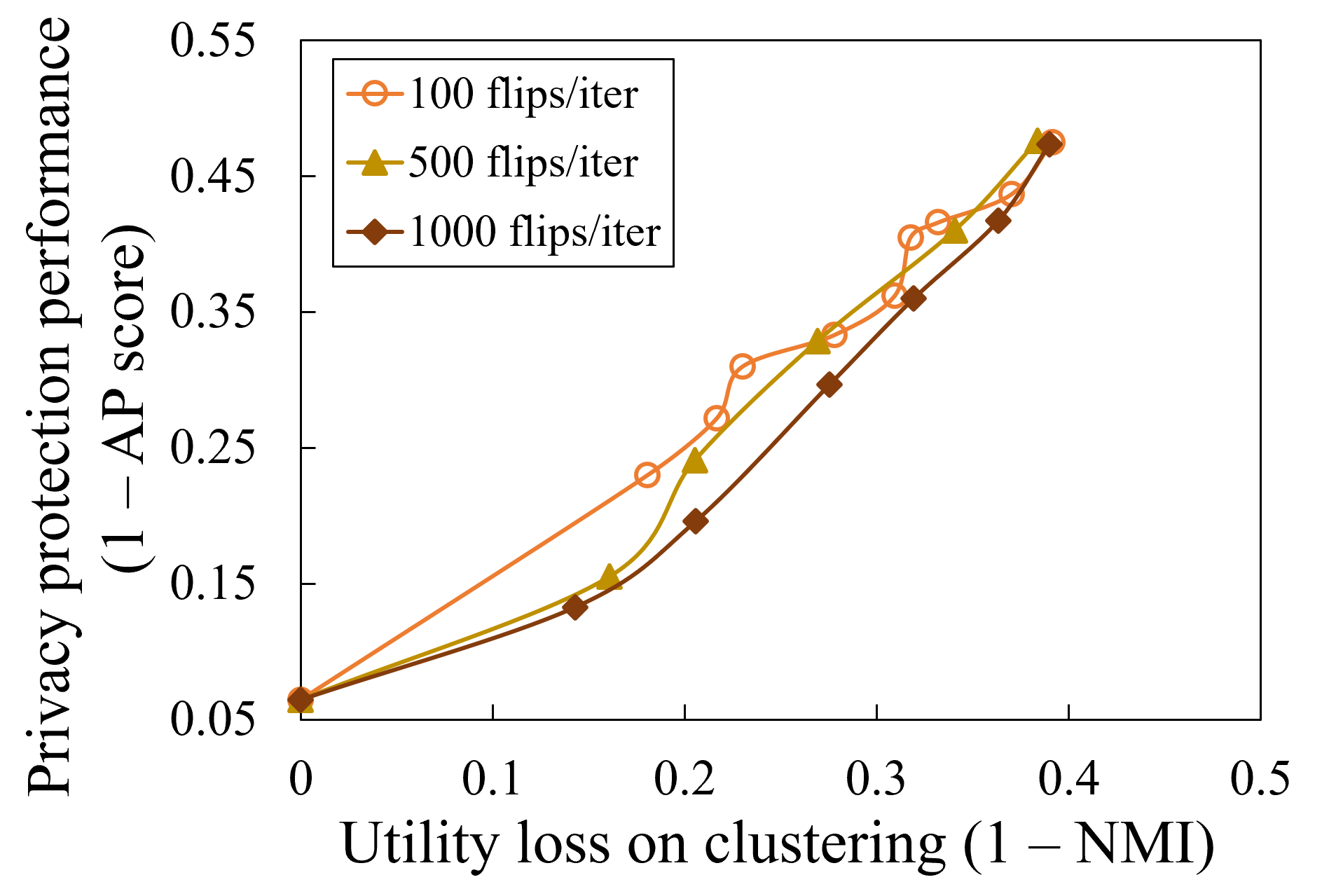}
		\caption{Node Clustering}
		\label{fig:flipnumSICL}
	\end{subfigure}
	\vspace{-1em}
	\caption{Impacts of $f$ (DeepWalk)}\label{fig:FlipNumParam}
	\end{minipage}
	\vspace{-1em}
\end{figure*}

\begin{figure*}[t]
	\centering
	\begin{minipage}{.25\textwidth}
    	\includegraphics[width=1\textwidth]{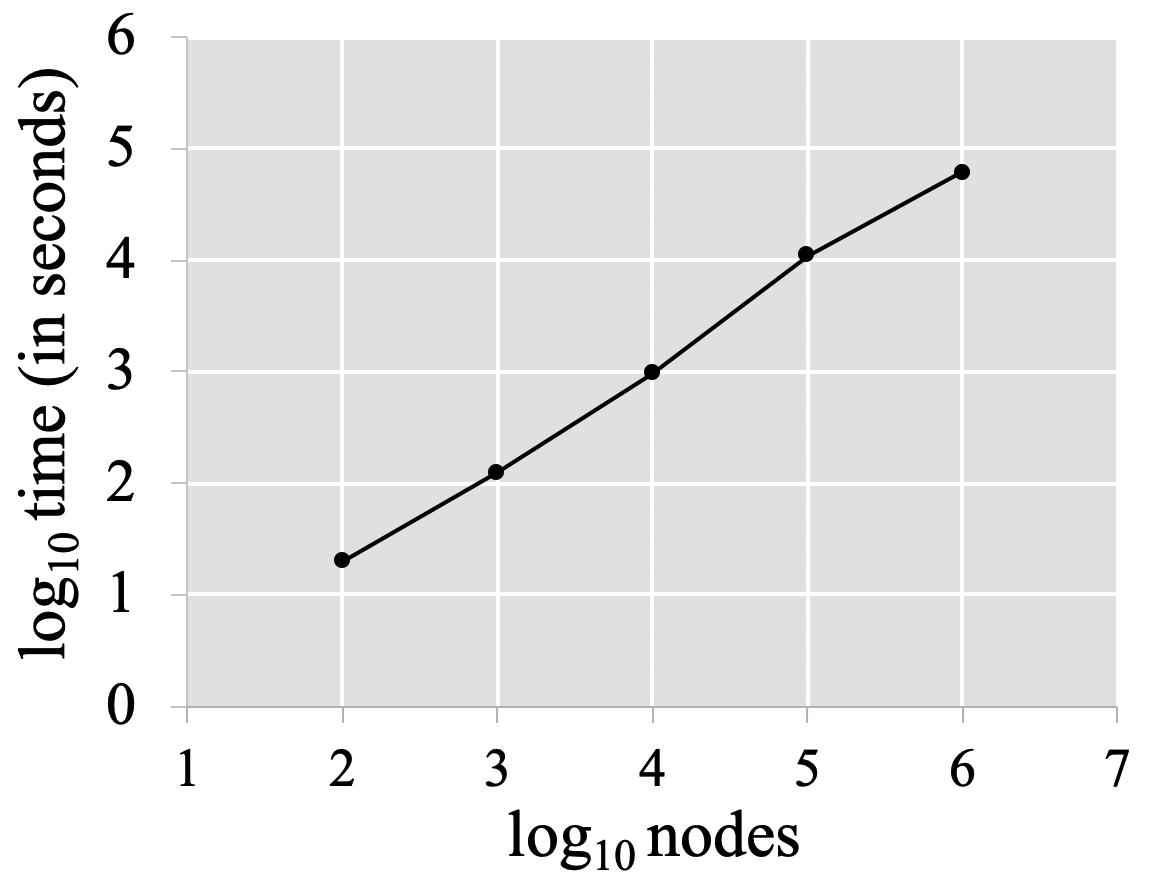}
	    \vspace{-2em}
    	\caption{Scalability of PPNE}
    	\label{fig:scalability_node}
	\end{minipage}
    \ \ \ \ \ \ \ \ 
	\begin{minipage}{.6\textwidth}
    	\begin{subfigure}[b]{0.45\textwidth}
    		\includegraphics[width=1\textwidth]{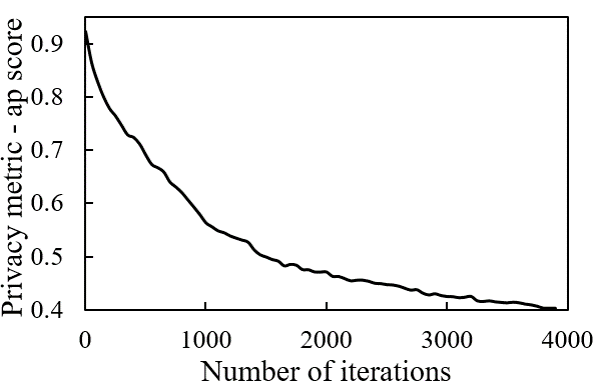}
    		\caption{Privacy metric}
    		\label{fig:convergence_priv}
    	\end{subfigure}
    	~\ \ \ \ 
    	\begin{subfigure}[b]{0.49\textwidth}
    		\includegraphics[width=1\textwidth]{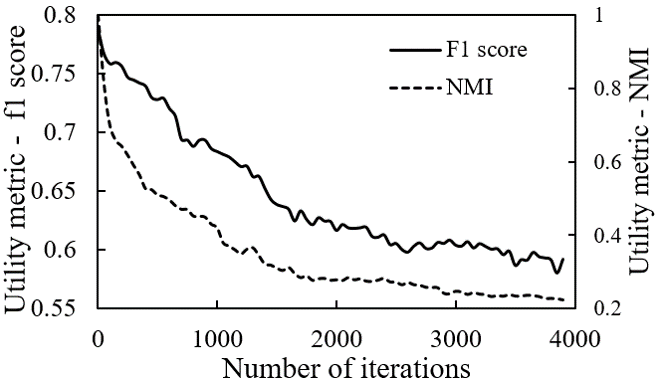}
    		\caption{Utility metrics}
    		\label{fig:convergence_util_nc}
    	\end{subfigure}
	\vspace{-1em}
    \caption{Convergence of PPNE}
	\label{fig:convergence}
	\end{minipage}
	\vspace{-1em}
\end{figure*}

\textbf{Perturbation sampling size} ($s$).
We vary the perturbation sampling size to investigate how the number of candidate perturbations involved in optimization at each iteration would impact PPNE. Specifically, we evaluate three values of sampling size (10k, 50k, and 100k) on PubMed. As expected, the effectiveness of PPNE increases by enlarging the sampling size, since the larger sampling size means more candidate perturbations considered during optimization. 

\textbf{Perturbation batch size} ($f$).
The parameter $f$ is varied to investigate how the number of perturbations conducted in each iteration would impact the performance of PPNE. Fig. \ref{fig:FlipNumParam} shows that a smaller $f$ leads to a better privacy protection performance. The reason is that a smaller $f$ indicates more accurate utility loss and privacy gain computation for candidate perturbations. However, it takes more iterations to reach a target privacy level for a smaller $f$.

\subsubsection{Scalability}
Our work has put much effort in enhancing the scalability of PPNE, and have shown that PPNE can work well for a large network including $1.7$ millions of nodes (Flickr). In this section, we further investigate the scalability of PPNE, and show the time of one iteration of PPNE with the simulated networks of different size in Fig.~\ref{fig:scalability_node}. In particular, we simulate a number of Erdos-Renyi networks with different node size from 100 to 1,000,000 nodes and constant average degree of 10. The default perturbation sampling size is set to 10,000. In general, we observe that the running time of PPNE increases linearly with the number of network nodes. Specifically, for the network with up to 1,000,000 nodes, we can finish one iteration in 3 hours. 

\subsubsection{Convergence}
In this part, we verify the convergence of PPNE empirically. As the number of iterations increases, more perturbations are carried out on the network, which is supposed to decrease the privacy leakage and utility of PPNE. We perform PPNE on Cora with default parameter values and plot the changes of the privacy and utility metrics by increasing the number of iterations in  Fig.~\ref{fig:convergence}. Just as expected, we observe that both of the privacy leakage and utility of PPNE decrease quickly in an iteration when the number of iterations is small; and the decrease get slower and slower as the number of iterations increases; the privacy leakage and utility of PPNE finally become converged at approximately the 3000th iteration (here we do one perturbation per iteration).

\section{Related Work}

Publishing data directly without anonymization may lead to severe privacy leakage~\cite{gong2016you, jia2017attriinfer}. Traditional privacy-preserving network data publishing mainly focuses on tabular data and network structure data. When publishing tabular data, generalization techniques~\cite{xu2014a}, obfuscation techniques \cite{salamatian2015managing} and adding noise~\cite{jia2018attriguard} are commonly used methods. As for sharing network graph data, differential privacy is the most popular approach which generates synthetic DP-networks with either edge differential privacy~\cite{hay2009accurate} or node differential privacy~\cite{kasiviswanathan2013analyzing}. Other works protect user privacy by adding or deleting the links on network~\cite{chen2018disclose, yu2019target}.

Recently, network embedding has become one of the most important and ubiquitous network data representation methods~\cite{hamilton2017representation}. However, people often overlook the privacy threats when publishing network embeddings, which may result in leakage of users' private information including attributes, membership and links~\cite{duddu2020quantifying}. This work focuses on privacy-preserving network embedding publishing techniques against private link inference attacks.

Although the privacy leakage problem when publishing network embedding does exist, there are few researches working on it. Recently, a differential privacy framework named DPNE~\cite{xu2018dpne} has been proposed to generate network embedding with theoretical privacy guarantees, which ensures the embeddings generated by two arbitrary adjacency networks (differ on a single link) indisguishable in probability via adding random noises. As DPNE is not designed to defend against private link inference attacks, the attackers with differentially private network embedding still can accurately private links of the network~\cite{jia2018attriguard, shokri2015privacy}. GNNs with adversarial training are also used to generate privacy-preserving network representations~\cite{li2020adversarial, wang2021privacy}. Specifically, these methods aim to produce the network embedding that is only useful for a specific task but effective for another task. In comparison, our work proposes a practical privacy-preserving mechanism for large-scale network embedding publishing, which can efficiently prevent private link inference attacks and benefit various downstream tasks.
 
\section{Conclusion and Discussion}

This work addresses a new problem of privacy-preserving network embedding against private link inference attacks. We first perturb the original network by adding/removing links, and then learn the embedding on the perturbed network. In order to find the optimal perturbed network, we proposed PPNE, which incorporates novel utility loss and privacy gain computation methods to quantify perturbations, and then perturbs the original network iteratively to reach satisfactory privacy-utility tradeoff.

We discuss our limitations and future work. (i) PPNE is suitable for undirected networks and how to extend it to directed networks remains a question. (ii) The scalability improvement of PPNE is currently designed for skip-gram network embedding methods (DeepWalk and LINE); how to improve its scalability on other network embedding methods would be another future direction. (iii) While PPNE aims to protect private links on the network, other private information (e.g. private attributes) may also be taken into consideration in the future work. 

\section*{Acknowledgment}

Authors contribute equally and are listed in alphabetical order.


\bibliographystyle{ACM-Reference-Format}
\bibliography{related-works}


\begin{thebibliography}{42}


\ifx \showCODEN    \undefined \def \showCODEN     #1{\unskip}     \fi
\ifx \showDOI      \undefined \def \showDOI       #1{#1}\fi
\ifx \showISBNx    \undefined \def \showISBNx     #1{\unskip}     \fi
\ifx \showISBNxiii \undefined \def \showISBNxiii  #1{\unskip}     \fi
\ifx \showISSN     \undefined \def \showISSN      #1{\unskip}     \fi
\ifx \showLCCN     \undefined \def \showLCCN      #1{\unskip}     \fi
\ifx \shownote     \undefined \def \shownote      #1{#1}          \fi
\ifx \showarticletitle \undefined \def \showarticletitle #1{#1}   \fi
\ifx \showURL      \undefined \def \showURL       {\relax}        \fi
\providecommand\bibfield[2]{#2}
\providecommand\bibinfo[2]{#2}
\providecommand\natexlab[1]{#1}
\providecommand\showeprint[2][]{arXiv:#2}

\bibitem[\protect\citeauthoryear{Bojchevski and G{\"u}nnemann}{Bojchevski and
  G{\"u}nnemann}{2019}]%
        {bojchevski2019adversarial}
\bibfield{author}{\bibinfo{person}{Aleksandar Bojchevski} {and}
  \bibinfo{person}{Stephan G{\"u}nnemann}.} \bibinfo{year}{2019}\natexlab{}.
\newblock \showarticletitle{Adversarial attacks on node embeddings via graph
  poisoning}. In \bibinfo{booktitle}{\emph{International Conference on Machine
  Learning}}. PMLR, \bibinfo{pages}{695--704}.
\newblock


\bibitem[\protect\citeauthoryear{Chang, Rong, Xu, Huang, Zhang, Cui, Zhu, and
  Huang}{Chang et~al\mbox{.}}{2020}]%
        {chang2020restricted}
\bibfield{author}{\bibinfo{person}{Heng Chang}, \bibinfo{person}{Yu Rong},
  \bibinfo{person}{Tingyang Xu}, \bibinfo{person}{Wenbing Huang},
  \bibinfo{person}{Honglei Zhang}, \bibinfo{person}{Peng Cui},
  \bibinfo{person}{Wenwu Zhu}, {and} \bibinfo{person}{Junzhou Huang}.}
  \bibinfo{year}{2020}\natexlab{}.
\newblock \showarticletitle{A Restricted Black-box Adversarial Framework
  Towards Attacking Graph Embedding Models}. AAAI.
\newblock


\bibitem[\protect\citeauthoryear{Chen, He, Cai, and Pan}{Chen
  et~al\mbox{.}}{2018}]%
        {chen2018disclose}
\bibfield{author}{\bibinfo{person}{Jiayi Chen}, \bibinfo{person}{Jianping He},
  \bibinfo{person}{Lin Cai}, {and} \bibinfo{person}{Jianping Pan}.}
  \bibinfo{year}{2018}\natexlab{}.
\newblock \showarticletitle{Disclose more and risk less: Privacy preserving
  online social network data sharing}.
\newblock \bibinfo{journal}{\emph{IEEE Transactions on Dependable and Secure
  Computing}} (\bibinfo{year}{2018}).
\newblock


\bibitem[\protect\citeauthoryear{Chen, He, Cai, and Pan}{Chen
  et~al\mbox{.}}{2020}]%
        {Chen2020DiscloseMA}
\bibfield{author}{\bibinfo{person}{Jiayi Chen}, \bibinfo{person}{Jianping He},
  \bibinfo{person}{Lin~X. Cai}, {and} \bibinfo{person}{Jianping Pan}.}
  \bibinfo{year}{2020}\natexlab{}.
\newblock \showarticletitle{Disclose More and Risk Less: Privacy Preserving
  Online Social Network Data Sharing}.
\newblock \bibinfo{journal}{\emph{IEEE Transactions on Dependable and Secure
  Computing}}  \bibinfo{volume}{17} (\bibinfo{year}{2020}),
  \bibinfo{pages}{1173--1187}.
\newblock


\bibitem[\protect\citeauthoryear{Chen and Guestrin}{Chen and Guestrin}{2016}]%
        {chen2016xgboost}
\bibfield{author}{\bibinfo{person}{Tianqi Chen} {and} \bibinfo{person}{Carlos
  Guestrin}.} \bibinfo{year}{2016}\natexlab{}.
\newblock \showarticletitle{Xgboost: A scalable tree boosting system}. In
  \bibinfo{booktitle}{\emph{Proceedings of the 22nd acm sigkdd international
  conference on knowledge discovery and data mining}}.
  \bibinfo{pages}{785--794}.
\newblock


\bibitem[\protect\citeauthoryear{Cheng, Fu, and Liu}{Cheng
  et~al\mbox{.}}{2010}]%
        {cheng2010k}
\bibfield{author}{\bibinfo{person}{James Cheng}, \bibinfo{person}{Ada Wai-chee
  Fu}, {and} \bibinfo{person}{Jia Liu}.} \bibinfo{year}{2010}\natexlab{}.
\newblock \showarticletitle{K-isomorphism: privacy preserving network
  publication against structural attacks}. In
  \bibinfo{booktitle}{\emph{Proceedings of the 2010 ACM SIGMOD International
  Conference on Management of data}}. \bibinfo{pages}{459--470}.
\newblock


\bibitem[\protect\citeauthoryear{Dai, Li, Tian, Huang, Wang, Zhu, and Song}{Dai
  et~al\mbox{.}}{2018}]%
        {dai2018adversarial}
\bibfield{author}{\bibinfo{person}{Hanjun Dai}, \bibinfo{person}{Hui Li},
  \bibinfo{person}{Tian Tian}, \bibinfo{person}{Xin Huang},
  \bibinfo{person}{Lin Wang}, \bibinfo{person}{Jun Zhu}, {and}
  \bibinfo{person}{Le Song}.} \bibinfo{year}{2018}\natexlab{}.
\newblock \showarticletitle{Adversarial Attack on Graph Structured Data}. In
  \bibinfo{booktitle}{\emph{International Conference on Machine Learning}}.
  \bibinfo{pages}{1115--1124}.
\newblock


\bibitem[\protect\citeauthoryear{Duddu, Boutet, and Shejwalkar}{Duddu
  et~al\mbox{.}}{2020}]%
        {duddu2020quantifying}
\bibfield{author}{\bibinfo{person}{Vasisht Duddu}, \bibinfo{person}{Antoine
  Boutet}, {and} \bibinfo{person}{Virat Shejwalkar}.}
  \bibinfo{year}{2020}\natexlab{}.
\newblock \showarticletitle{Quantifying Privacy Leakage in Graph Embedding}.
\newblock \bibinfo{journal}{\emph{arXiv preprint arXiv:2010.00906}}
  (\bibinfo{year}{2020}).
\newblock


\bibitem[\protect\citeauthoryear{Fung, Wang, Chen, and Yu}{Fung
  et~al\mbox{.}}{2010}]%
        {fung2010privacy}
\bibfield{author}{\bibinfo{person}{Benjamin~CM Fung}, \bibinfo{person}{Ke
  Wang}, \bibinfo{person}{Rui Chen}, {and} \bibinfo{person}{Philip~S Yu}.}
  \bibinfo{year}{2010}\natexlab{}.
\newblock \showarticletitle{Privacy-preserving data publishing: A survey of
  recent developments}.
\newblock \bibinfo{journal}{\emph{ACM Computing Surveys (Csur)}}
  \bibinfo{volume}{42}, \bibinfo{number}{4} (\bibinfo{year}{2010}),
  \bibinfo{pages}{1--53}.
\newblock


\bibitem[\protect\citeauthoryear{{Gong} and {Liu}}{{Gong} and {Liu}}{2016}]%
        {gong2016you}
\bibfield{author}{\bibinfo{person}{Neil~Zhenqiang {Gong}} {and}
  \bibinfo{person}{Bin {Liu}}.} \bibinfo{year}{2016}\natexlab{}.
\newblock \showarticletitle{You are who you know and how you behave: attribute
  inference attacks via users' social friends and behaviors}. In
  \bibinfo{booktitle}{\emph{SEC'16 Proceedings of the 25th USENIX Conference on
  Security Symposium}}. \bibinfo{pages}{979--995}.
\newblock


\bibitem[\protect\citeauthoryear{Goyal and Ferrara}{Goyal and Ferrara}{2018}]%
        {goyal2018graph}
\bibfield{author}{\bibinfo{person}{Palash Goyal} {and} \bibinfo{person}{Emilio
  Ferrara}.} \bibinfo{year}{2018}\natexlab{}.
\newblock \showarticletitle{Graph embedding techniques, applications, and
  performance: A survey}.
\newblock \bibinfo{journal}{\emph{Knowledge-Based Systems}}
  \bibinfo{volume}{151} (\bibinfo{year}{2018}), \bibinfo{pages}{78--94}.
\newblock


\bibitem[\protect\citeauthoryear{Grover and Leskovec}{Grover and
  Leskovec}{2016}]%
        {grover2016node2vec}
\bibfield{author}{\bibinfo{person}{Aditya Grover} {and} \bibinfo{person}{Jure
  Leskovec}.} \bibinfo{year}{2016}\natexlab{}.
\newblock \showarticletitle{node2vec: Scalable feature learning for networks}.
  In \bibinfo{booktitle}{\emph{Proceedings of the 22nd ACM SIGKDD international
  conference on Knowledge discovery and data mining}}.
  \bibinfo{pages}{855--864}.
\newblock


\bibitem[\protect\citeauthoryear{Hamilton, Ying, and Leskovec}{Hamilton
  et~al\mbox{.}}{2017}]%
        {hamilton2017representation}
\bibfield{author}{\bibinfo{person}{William~L Hamilton}, \bibinfo{person}{Rex
  Ying}, {and} \bibinfo{person}{Jure Leskovec}.}
  \bibinfo{year}{2017}\natexlab{}.
\newblock \showarticletitle{Representation learning on graphs: Methods and
  applications}.
\newblock \bibinfo{journal}{\emph{arXiv preprint arXiv:1709.05584}}
  (\bibinfo{year}{2017}).
\newblock


\bibitem[\protect\citeauthoryear{Hay, Li, Miklau, and Jensen}{Hay
  et~al\mbox{.}}{2009}]%
        {hay2009accurate}
\bibfield{author}{\bibinfo{person}{Michael Hay}, \bibinfo{person}{Chao Li},
  \bibinfo{person}{Gerome Miklau}, {and} \bibinfo{person}{David Jensen}.}
  \bibinfo{year}{2009}\natexlab{}.
\newblock \showarticletitle{Accurate estimation of the degree distribution of
  private networks}. In \bibinfo{booktitle}{\emph{2009 Ninth IEEE International
  Conference on Data Mining}}. IEEE, \bibinfo{pages}{169--178}.
\newblock


\bibitem[\protect\citeauthoryear{Jayaraman and Evans}{Jayaraman and
  Evans}{2019}]%
        {jayaraman2019evaluating}
\bibfield{author}{\bibinfo{person}{Bargav Jayaraman} {and}
  \bibinfo{person}{David Evans}.} \bibinfo{year}{2019}\natexlab{}.
\newblock \showarticletitle{Evaluating differentially private machine learning
  in practice}. In \bibinfo{booktitle}{\emph{28th $\{$USENIX$\}$ Security
  Symposium ($\{$USENIX$\}$ Security 19)}}. \bibinfo{pages}{1895--1912}.
\newblock


\bibitem[\protect\citeauthoryear{Jia and Gong}{Jia and Gong}{2018}]%
        {jia2018attriguard}
\bibfield{author}{\bibinfo{person}{Jinyuan Jia} {and}
  \bibinfo{person}{Neil~Zhenqiang Gong}.} \bibinfo{year}{2018}\natexlab{}.
\newblock \showarticletitle{Attriguard: A practical defense against attribute
  inference attacks via adversarial machine learning}. In
  \bibinfo{booktitle}{\emph{27th $\{$USENIX$\}$ Security Symposium
  ($\{$USENIX$\}$ Security 18)}}. \bibinfo{pages}{513--529}.
\newblock


\bibitem[\protect\citeauthoryear{{Jia}, {Wang}, {Zhang}, and {Gong}}{{Jia}
  et~al\mbox{.}}{2017}]%
        {jia2017attriinfer}
\bibfield{author}{\bibinfo{person}{Jinyuan {Jia}}, \bibinfo{person}{Binghui
  {Wang}}, \bibinfo{person}{Le {Zhang}}, {and} \bibinfo{person}{Neil~Zhenqiang
  {Gong}}.} \bibinfo{year}{2017}\natexlab{}.
\newblock \showarticletitle{AttriInfer: Inferring User Attributes in Online
  Social Networks Using Markov Random Fields}. In \bibinfo{booktitle}{\emph{WWW
  '17 Proceedings of the 26th International Conference on World Wide Web}}.
  \bibinfo{pages}{1561--1569}.
\newblock


\bibitem[\protect\citeauthoryear{Kasiviswanathan, Nissim, Raskhodnikova, and
  Smith}{Kasiviswanathan et~al\mbox{.}}{2013}]%
        {kasiviswanathan2013analyzing}
\bibfield{author}{\bibinfo{person}{Shiva~Prasad Kasiviswanathan},
  \bibinfo{person}{Kobbi Nissim}, \bibinfo{person}{Sofya Raskhodnikova}, {and}
  \bibinfo{person}{Adam Smith}.} \bibinfo{year}{2013}\natexlab{}.
\newblock \showarticletitle{Analyzing graphs with node differential privacy}.
  In \bibinfo{booktitle}{\emph{Theory of Cryptography Conference}}. Springer,
  \bibinfo{pages}{457--476}.
\newblock


\bibitem[\protect\citeauthoryear{Levy, Goldberg, and Dagan}{Levy
  et~al\mbox{.}}{2015}]%
        {levy2015improving}
\bibfield{author}{\bibinfo{person}{Omer Levy}, \bibinfo{person}{Yoav Goldberg},
  {and} \bibinfo{person}{Ido Dagan}.} \bibinfo{year}{2015}\natexlab{}.
\newblock \showarticletitle{Improving distributional similarity with lessons
  learned from word embeddings}.
\newblock \bibinfo{journal}{\emph{Transactions of the Association for
  Computational Linguistics}}  \bibinfo{volume}{3} (\bibinfo{year}{2015}),
  \bibinfo{pages}{211--225}.
\newblock


\bibitem[\protect\citeauthoryear{Li, Luo, Ye, Li, Ji, and Cai}{Li
  et~al\mbox{.}}{2020}]%
        {li2020adversarial}
\bibfield{author}{\bibinfo{person}{Kaiyang Li}, \bibinfo{person}{Guangchun
  Luo}, \bibinfo{person}{Yang Ye}, \bibinfo{person}{Wei Li},
  \bibinfo{person}{Shihao Ji}, {and} \bibinfo{person}{Zhipeng Cai}.}
  \bibinfo{year}{2020}\natexlab{}.
\newblock \showarticletitle{Adversarial Privacy Preserving Graph Embedding
  against Inference Attack}.
\newblock \bibinfo{journal}{\emph{IEEE Internet of Things Journal}}
  (\bibinfo{year}{2020}).
\newblock


\bibitem[\protect\citeauthoryear{Liao, Zhao, Xu, Jaakkola, Gordon, Jegelka, and
  Salakhutdinov}{Liao et~al\mbox{.}}{2021}]%
        {pmlr-v139-liao21a}
\bibfield{author}{\bibinfo{person}{Peiyuan Liao}, \bibinfo{person}{Han Zhao},
  \bibinfo{person}{Keyulu Xu}, \bibinfo{person}{Tommi Jaakkola},
  \bibinfo{person}{Geoffrey~J. Gordon}, \bibinfo{person}{Stefanie Jegelka},
  {and} \bibinfo{person}{Ruslan Salakhutdinov}.}
  \bibinfo{year}{2021}\natexlab{}.
\newblock \showarticletitle{Information Obfuscation of Graph Neural Networks}.
  In \bibinfo{booktitle}{\emph{ICML}}, Vol.~\bibinfo{volume}{139}.
  \bibinfo{pages}{6600--6610}.
\newblock


\bibitem[\protect\citeauthoryear{L{\"u} and Zhou}{L{\"u} and Zhou}{2011}]%
        {lu2011link}
\bibfield{author}{\bibinfo{person}{Linyuan L{\"u}} {and} \bibinfo{person}{Tao
  Zhou}.} \bibinfo{year}{2011}\natexlab{}.
\newblock \showarticletitle{Link prediction in complex networks: A survey}.
\newblock \bibinfo{journal}{\emph{Physica A: statistical mechanics and its
  applications}} \bibinfo{volume}{390}, \bibinfo{number}{6}
  (\bibinfo{year}{2011}), \bibinfo{pages}{1150--1170}.
\newblock


\bibitem[\protect\citeauthoryear{Ma, Ding, and Mei}{Ma et~al\mbox{.}}{2020}]%
        {ma2020towards}
\bibfield{author}{\bibinfo{person}{Jiaqi Ma}, \bibinfo{person}{Shuangrui Ding},
  {and} \bibinfo{person}{Qiaozhu Mei}.} \bibinfo{year}{2020}\natexlab{}.
\newblock \showarticletitle{Towards More Practical Adversarial Attacks on Graph
  Neural Networks}.
\newblock \bibinfo{journal}{\emph{Advances in Neural Information Processing
  Systems}}  \bibinfo{volume}{33} (\bibinfo{year}{2020}).
\newblock


\bibitem[\protect\citeauthoryear{Ou, Cui, Pei, Zhang, and Zhu}{Ou
  et~al\mbox{.}}{2016}]%
        {ou2016asymmetric}
\bibfield{author}{\bibinfo{person}{Mingdong Ou}, \bibinfo{person}{Peng Cui},
  \bibinfo{person}{Jian Pei}, \bibinfo{person}{Ziwei Zhang}, {and}
  \bibinfo{person}{Wenwu Zhu}.} \bibinfo{year}{2016}\natexlab{}.
\newblock \showarticletitle{Asymmetric transitivity preserving graph
  embedding}. In \bibinfo{booktitle}{\emph{Proceedings of the 22nd ACM SIGKDD
  international conference on Knowledge discovery and data mining}}.
  \bibinfo{pages}{1105--1114}.
\newblock


\bibitem[\protect\citeauthoryear{Perozzi, Al-Rfou, and Skiena}{Perozzi
  et~al\mbox{.}}{2014}]%
        {perozzi2014deepwalk}
\bibfield{author}{\bibinfo{person}{Bryan Perozzi}, \bibinfo{person}{Rami
  Al-Rfou}, {and} \bibinfo{person}{Steven Skiena}.}
  \bibinfo{year}{2014}\natexlab{}.
\newblock \showarticletitle{Deepwalk: Online learning of social
  representations}. In \bibinfo{booktitle}{\emph{Proceedings of the 20th ACM
  SIGKDD international conference on Knowledge discovery and data mining}}.
  \bibinfo{pages}{701--710}.
\newblock


\bibitem[\protect\citeauthoryear{Qiu, Dong, Ma, Li, Wang, and Tang}{Qiu
  et~al\mbox{.}}{2018}]%
        {qiu2018network}
\bibfield{author}{\bibinfo{person}{Jiezhong Qiu}, \bibinfo{person}{Yuxiao
  Dong}, \bibinfo{person}{Hao Ma}, \bibinfo{person}{Jian Li},
  \bibinfo{person}{Kuansan Wang}, {and} \bibinfo{person}{Jie Tang}.}
  \bibinfo{year}{2018}\natexlab{}.
\newblock \showarticletitle{Network embedding as matrix factorization: Unifying
  deepwalk, line, pte, and node2vec}. In \bibinfo{booktitle}{\emph{Proceedings
  of the Eleventh ACM International Conference on Web Search and Data Mining}}.
  \bibinfo{pages}{459--467}.
\newblock


\bibitem[\protect\citeauthoryear{Sala, Zhao, Wilson, Zheng, and Zhao}{Sala
  et~al\mbox{.}}{2011}]%
        {sala2011sharing}
\bibfield{author}{\bibinfo{person}{Alessandra Sala}, \bibinfo{person}{Xiaohan
  Zhao}, \bibinfo{person}{Christo Wilson}, \bibinfo{person}{Haitao Zheng},
  {and} \bibinfo{person}{Ben~Y Zhao}.} \bibinfo{year}{2011}\natexlab{}.
\newblock \showarticletitle{Sharing graphs using differentially private graph
  models}. In \bibinfo{booktitle}{\emph{Proceedings of the 2011 ACM SIGCOMM
  conference on Internet measurement conference}}. \bibinfo{pages}{81--98}.
\newblock


\bibitem[\protect\citeauthoryear{{Salamatian}, {Zhang}, du~Pin~{Calmon},
  {Bhamidipati}, {Fawaz}, {Kveton}, {Oliveira}, and {Taft}}{{Salamatian}
  et~al\mbox{.}}{2015}]%
        {salamatian2015managing}
\bibfield{author}{\bibinfo{person}{Salman {Salamatian}}, \bibinfo{person}{Amy
  {Zhang}}, \bibinfo{person}{Flavio du Pin~{Calmon}}, \bibinfo{person}{Sandilya
  {Bhamidipati}}, \bibinfo{person}{Nadia {Fawaz}}, \bibinfo{person}{Branislav
  {Kveton}}, \bibinfo{person}{Pedro {Oliveira}}, {and} \bibinfo{person}{Nina
  {Taft}}.} \bibinfo{year}{2015}\natexlab{}.
\newblock \showarticletitle{Managing Your Private and Public Data: Bringing
  Down Inference Attacks Against Your Privacy}.
\newblock \bibinfo{journal}{\emph{IEEE Journal of Selected Topics in Signal
  Processing}} \bibinfo{volume}{9}, \bibinfo{number}{7} (\bibinfo{year}{2015}),
  \bibinfo{pages}{1240--1255}.
\newblock


\bibitem[\protect\citeauthoryear{Sen, Namata, Bilgic, Getoor, Galligher, and
  Eliassi-Rad}{Sen et~al\mbox{.}}{2008}]%
        {sen2008collective}
\bibfield{author}{\bibinfo{person}{Prithviraj Sen}, \bibinfo{person}{Galileo
  Namata}, \bibinfo{person}{Mustafa Bilgic}, \bibinfo{person}{Lise Getoor},
  \bibinfo{person}{Brian Galligher}, {and} \bibinfo{person}{Tina Eliassi-Rad}.}
  \bibinfo{year}{2008}\natexlab{}.
\newblock \showarticletitle{Collective classification in network data}.
\newblock \bibinfo{journal}{\emph{AI magazine}} \bibinfo{volume}{29},
  \bibinfo{number}{3} (\bibinfo{year}{2008}), \bibinfo{pages}{93--93}.
\newblock


\bibitem[\protect\citeauthoryear{Shokri}{Shokri}{2015}]%
        {shokri2015privacy}
\bibfield{author}{\bibinfo{person}{Reza Shokri}.}
  \bibinfo{year}{2015}\natexlab{}.
\newblock \showarticletitle{Privacy Games: Optimal User-Centric Data
  Obfuscation.}
\newblock \bibinfo{journal}{\emph{Proc. Priv. Enhancing Technol.}}
  \bibinfo{volume}{2015}, \bibinfo{number}{2} (\bibinfo{year}{2015}),
  \bibinfo{pages}{299--315}.
\newblock


\bibitem[\protect\citeauthoryear{Sun, Tang, Li, Li, Xiao, Chen, and Song}{Sun
  et~al\mbox{.}}{2018}]%
        {sun2018data}
\bibfield{author}{\bibinfo{person}{Mingjie Sun}, \bibinfo{person}{Jian Tang},
  \bibinfo{person}{Huichen Li}, \bibinfo{person}{Bo Li},
  \bibinfo{person}{Chaowei Xiao}, \bibinfo{person}{Yao Chen}, {and}
  \bibinfo{person}{Dawn Song}.} \bibinfo{year}{2018}\natexlab{}.
\newblock \showarticletitle{Data poisoning attack against unsupervised node
  embedding methods}.
\newblock \bibinfo{journal}{\emph{arXiv preprint arXiv:1810.12881}}
  (\bibinfo{year}{2018}).
\newblock


\bibitem[\protect\citeauthoryear{Tang, Qu, Wang, Zhang, Yan, and Mei}{Tang
  et~al\mbox{.}}{2015}]%
        {tang2015line}
\bibfield{author}{\bibinfo{person}{Jian Tang}, \bibinfo{person}{Meng Qu},
  \bibinfo{person}{Mingzhe Wang}, \bibinfo{person}{Ming Zhang},
  \bibinfo{person}{Jun Yan}, {and} \bibinfo{person}{Qiaozhu Mei}.}
  \bibinfo{year}{2015}\natexlab{}.
\newblock \showarticletitle{Line: Large-scale information network embedding}.
  In \bibinfo{booktitle}{\emph{Proceedings of the 24th international conference
  on world wide web}}. \bibinfo{pages}{1067--1077}.
\newblock


\bibitem[\protect\citeauthoryear{Wang, Guo, Li, Chen, and Li}{Wang
  et~al\mbox{.}}{2021}]%
        {wang2021privacy}
\bibfield{author}{\bibinfo{person}{Binghui Wang}, \bibinfo{person}{Jiayi Guo},
  \bibinfo{person}{Ang Li}, \bibinfo{person}{Yiran Chen}, {and}
  \bibinfo{person}{Hai Li}.} \bibinfo{year}{2021}\natexlab{}.
\newblock \showarticletitle{Privacy-Preserving Representation Learning on
  Graphs: A Mutual Information Perspective}. In
  \bibinfo{booktitle}{\emph{Proceedings of the 27th ACM SIGKDD Conference on
  Knowledge Discovery \& Data Mining}}. \bibinfo{pages}{1667--1676}.
\newblock


\bibitem[\protect\citeauthoryear{Wang, Cui, and Zhu}{Wang
  et~al\mbox{.}}{2016}]%
        {wang2016structural}
\bibfield{author}{\bibinfo{person}{Daixin Wang}, \bibinfo{person}{Peng Cui},
  {and} \bibinfo{person}{Wenwu Zhu}.} \bibinfo{year}{2016}\natexlab{}.
\newblock \showarticletitle{Structural deep network embedding}. In
  \bibinfo{booktitle}{\emph{Proceedings of the 22nd ACM SIGKDD international
  conference on Knowledge discovery and data mining}}.
  \bibinfo{pages}{1225--1234}.
\newblock


\bibitem[\protect\citeauthoryear{White and Smyth}{White and Smyth}{2005}]%
        {white2005spectral}
\bibfield{author}{\bibinfo{person}{Scott White} {and} \bibinfo{person}{Padhraic
  Smyth}.} \bibinfo{year}{2005}\natexlab{}.
\newblock \showarticletitle{A spectral clustering approach to finding
  communities in graphs}. In \bibinfo{booktitle}{\emph{Proceedings of the 2005
  SIAM international conference on data mining}}. SIAM,
  \bibinfo{pages}{274--285}.
\newblock


\bibitem[\protect\citeauthoryear{Xiao, Chen, and Tan}{Xiao
  et~al\mbox{.}}{2014}]%
        {XiaoKDD2014}
\bibfield{author}{\bibinfo{person}{Qian Xiao}, \bibinfo{person}{Rui Chen},
  {and} \bibinfo{person}{Kian-Lee Tan}.} \bibinfo{year}{2014}\natexlab{}.
\newblock \showarticletitle{Differentially Private Network Data Release via
  Structural Inference}. In \bibinfo{booktitle}{\emph{KDD}}.
  \bibinfo{pages}{911–920}.
\newblock


\bibitem[\protect\citeauthoryear{Xu, Yuan, Wu, and Phan}{Xu
  et~al\mbox{.}}{2018}]%
        {xu2018dpne}
\bibfield{author}{\bibinfo{person}{Depeng Xu}, \bibinfo{person}{Shuhan Yuan},
  \bibinfo{person}{Xintao Wu}, {and} \bibinfo{person}{HaiNhat Phan}.}
  \bibinfo{year}{2018}\natexlab{}.
\newblock \showarticletitle{DPNE: Differentially private network embedding}. In
  \bibinfo{booktitle}{\emph{Pacific-Asia Conference on Knowledge Discovery and
  Data Mining}}. Springer, \bibinfo{pages}{235--246}.
\newblock


\bibitem[\protect\citeauthoryear{{Xu}, {Ma}, {Tang}, and {Tian}}{{Xu}
  et~al\mbox{.}}{2014}]%
        {xu2014a}
\bibfield{author}{\bibinfo{person}{Yang {Xu}}, \bibinfo{person}{Tinghuai {Ma}},
  \bibinfo{person}{Meili {Tang}}, {and} \bibinfo{person}{Wei {Tian}}.}
  \bibinfo{year}{2014}\natexlab{}.
\newblock \showarticletitle{A Survey of Privacy Preserving Data Publishing
  using Generalization and Suppression}.
\newblock \bibinfo{journal}{\emph{Applied Mathematics \& Information Sciences}}
  \bibinfo{volume}{8}, \bibinfo{number}{3} (\bibinfo{year}{2014}),
  \bibinfo{pages}{1103--1116}.
\newblock


\bibitem[\protect\citeauthoryear{{Yang}, {Qu}, and {Cudre-Mauroux}}{{Yang}
  et~al\mbox{.}}{2019}]%
        {yang2019privacy}
\bibfield{author}{\bibinfo{person}{Dingqi {Yang}}, \bibinfo{person}{Bingqing
  {Qu}}, {and} \bibinfo{person}{Philippe {Cudre-Mauroux}}.}
  \bibinfo{year}{2019}\natexlab{}.
\newblock \showarticletitle{Privacy-Preserving Social Media Data Publishing for
  Personalized Ranking-Based Recommendation}.
\newblock \bibinfo{journal}{\emph{IEEE Transactions on Knowledge and Data
  Engineering}} \bibinfo{volume}{31}, \bibinfo{number}{3}
  (\bibinfo{year}{2019}), \bibinfo{pages}{507--520}.
\newblock


\bibitem[\protect\citeauthoryear{Yu, Zhao, Fu, Zheng, Huang, Shu, Xuan, and
  Chen}{Yu et~al\mbox{.}}{2019}]%
        {yu2019target}
\bibfield{author}{\bibinfo{person}{Shanqing Yu}, \bibinfo{person}{Minghao
  Zhao}, \bibinfo{person}{Chenbo Fu}, \bibinfo{person}{Jun Zheng},
  \bibinfo{person}{Huimin Huang}, \bibinfo{person}{Xincheng Shu},
  \bibinfo{person}{Qi Xuan}, {and} \bibinfo{person}{G Chen}.}
  \bibinfo{year}{2019}\natexlab{}.
\newblock \showarticletitle{Target defense against link-prediction-based
  attacks via evolutionary perturbations}.
\newblock \bibinfo{journal}{\emph{IEEE Transactions on Knowledge and Data
  Engineering}} (\bibinfo{year}{2019}).
\newblock


\bibitem[\protect\citeauthoryear{Zhou, L{\"u}, and Zhang}{Zhou
  et~al\mbox{.}}{2009}]%
        {zhou2009predicting}
\bibfield{author}{\bibinfo{person}{Tao Zhou}, \bibinfo{person}{Linyuan L{\"u}},
  {and} \bibinfo{person}{Yi-Cheng Zhang}.} \bibinfo{year}{2009}\natexlab{}.
\newblock \showarticletitle{Predicting missing links via local information}.
\newblock \bibinfo{journal}{\emph{The European Physical Journal B}}
  \bibinfo{volume}{71}, \bibinfo{number}{4} (\bibinfo{year}{2009}),
  \bibinfo{pages}{623--630}.
\newblock


\bibitem[\protect\citeauthoryear{{Zügner} and {Günnemann}}{{Zügner} and
  {Günnemann}}{2019}]%
        {zugner2019adversarial}
\bibfield{author}{\bibinfo{person}{Daniel {Zügner}} {and}
  \bibinfo{person}{Stephan {Günnemann}}.} \bibinfo{year}{2019}\natexlab{}.
\newblock \showarticletitle{Adversarial Attacks on Graph Neural Networks via
  Meta Learning}. In \bibinfo{booktitle}{\emph{ICLR 2019 : 7th International
  Conference on Learning Representations}}.
\newblock


\end{thebibliography}

\section*{Appendix}
In this supplemental document, we provide the content related to the attackers with auxiliary knowledge, privacy-utility tradeoff results on Citeseer and PubMed dataset, and the experimental details.

\subsection*{Attackers with Auxiliary Knowledge}
For the sake of briefness, we simply call the attackers with no auxiliary knowledge Type-I attackers, and the attackers with auxiliary knowledge Type-II attackers. Type-II attackers hold auxiliary knowledge in addition to $X'$. In this work, we mainly consider that Type-II attackers foreknow the true relation status (1 or 0) of partial private node pairs $\mathcal P_\textit{know} \subset \mathcal P$. We also denote the rest relation-unknown sensitive node pairs for attacking as $\mathcal P_\textit{unknow}$ ($\mathcal P = \mathcal P_\textit{know} \cup \mathcal P_\textit{unknow}$). Type-II attackers then can adopt any supervised learning technique to construct a link prediction attack model $\mathcal A$ while $X'$ is the model input features.
\begin{equation}
	\textit{Attack}(v_i,v_j) = \mathcal A (X'_i, X'_j | \mathcal P_\textit{know}), \quad \forall (v_i,v_j) \in \mathcal P_\textit{unknow}
\end{equation}

Directly defending Type-II attackers is infeasible, as there are too many indeterministic factors such as the feature engineering, model training techniques, and the known sensitive node pairs $\mathcal P_\textit{know}$. 
In fact, cosine similarity based Type-I attack (Eq.~\ref{eq:attack_I_def}) can actually be seen as a special case of a strong embedding-based Type-II link prediction attack model~\cite{grover2016node2vec}. We believe that the defense on Type-I attacks can also help defend Type-II attacks to a certain extent.
\begin{theorem}
	\label{theorem:attack_1_to_attack_2}
	\textit{The cosine similarity based Type-I attack is a special case of the Hadamard-product embedding-based logistic regression Type-II attack model (which performs the best among all the embedding-based link prediction methods in literature \cite{grover2016node2vec}).}
\end{theorem}
\begin{proof}
	Denote $v_i$'s embedding vector as $X'_i =  \langle x'_{i,1}, x'_{i,2}, ..., x'_{i,m}\rangle$ and $X'_i$ is normalized. Then the cosine similarity of $X_i'$ and $X_j'$ can be written as:
	\begin{equation}
		sim(X'_i, X'_j) = x'_{i,1}x'_{j,1} + x'_{i,2}x'_{j,2} + ... + x'_{i,m}x'_{j,m}
		\label{eq:sim_cos}
	\end{equation}
	The prior research \cite{grover2016node2vec} has verified that the logistic regression on the Hadamard product of the embedding vectors of $v_i$ and $v_j$:
	\begin{equation}
		\mathcal A(X'_i, X'_j) = b_0 + b_1 x'_{i,1}x'_{j,1} + b_2 x'_{i,2}x'_{j,2} + ... + b_m x'_{i,m}x'_{j,m}
		\label{eq:logistic_regression}
	\end{equation}
	Eq.~\ref{eq:sim_cos} is thus a special case of Eq.~\ref{eq:logistic_regression} when node embeddings are normalized, $b_0 =0$, and $\forall_k b_k=1$.
\end{proof}

\textbf{Private link inference attacks with certain auxiliary knowledge.}
In this scenario, the private link inference attack problem can be regarded as a binary classification problem, where attackers concatenate two private nodes' embedding vectors as the input features, and whether a link exists between them or not (1 or 0) as the label. Attackers can train a classifier based on $\mathcal P_\textit{know}$ and use the classifier to predict unknown sensitive links. The performance of such link inference attack methods can be measured by $F1\ score$. We assess the privacy protection performance by $1 - F1\ score$. Higher $1 - F1\ score$ means better privacy protection. We use XGBoost \cite{chen2016xgboost} as the attack classifier model as it performs consistently well in our datasets.

\subsection*{Results on Citeseer and PubMed}
In this section, we compare the privacy-utility tradeoff performance of PPNE with baselines on on Citeseer and PubMed dataset. Results on Citeseer based on DeepWalk and LINE are shown in Fig. \ref{fig:CiteTradeOffDW} and Fig. \ref{fig:CiteTradeOffLINE}, while results on PubMed based on DeepWalk and LINE are illustrated in Fig. \ref{fig:PubTradeOffDW} and Fig. \ref{fig:PubTradeOffLINE}, respectively.

From Fig. \ref{fig:CiteTradeOffDW} to Fig. \ref{fig:PubTradeOffLINE}, we can observe the same trend that the privacy protection performances of all methods are becoming better with the decline in utility level, which indicates that we can improve privacy protection by sacrificing part of the data utility purposely. What's more, PPNE shows the best privacy-utility tradeoff performance compared to the baselines under all experimental settings by providing strong privacy protection with less utility loss.

In particular, as the results on Citeseer dataset illustrated in Fig. \ref{fig:CiteDWSINC} and Fig. \ref{fig:CiteDWSICL}, DICE achieves the best performance among the baselines, while PPNE improves the privacy protection by around 81.8\% and 7.1\% compared to it when the utility (1-f1 score) on node classification is around 0.52, respectively. On PubMed dataset, from the results shown in Fig. \ref{fig:PubDWSINC}, we note that when the utility on node classification task decreases around 5.2\% compared to the initial point, the privacy protection metric (1-ap score) brought by PPNE can be improved to around 0.24, while that of baselines only increased by around 0.01. These results on Citeseer and PubMed dataset are consistent with the reults on Cora dataset that PPNE achieves the best privacy-utility tradeoff among the methods both based on DeepWalk and LINE. And the results also prove that PPNE is applicable to multiple downstream utility applications including but not limited to node classification and clustering.

\begin{figure*}[b]
	\centering
	\begin{subfigure}[b]{0.23\textwidth}
		\includegraphics[width=\textwidth]{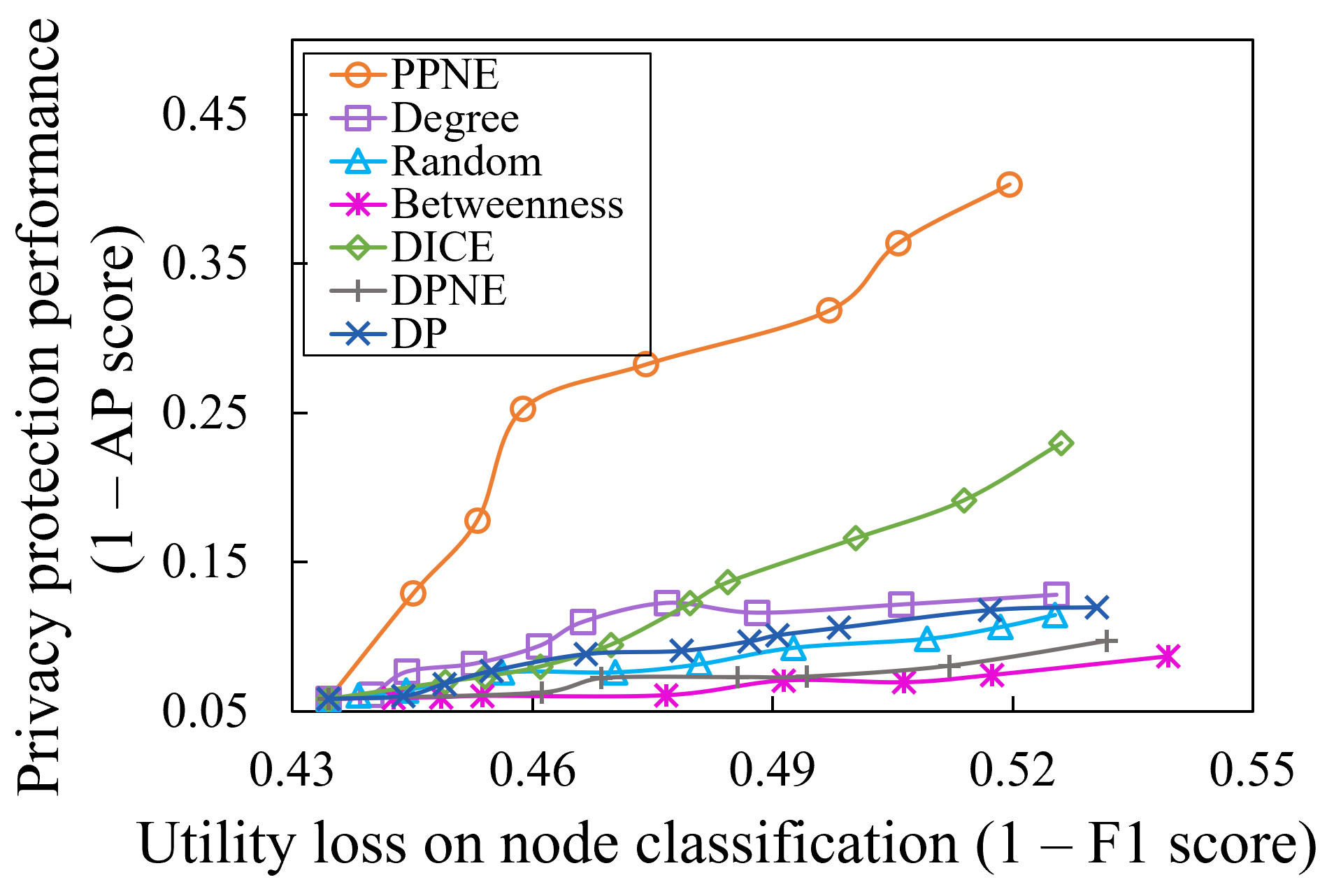}
		\caption{Scenario I on Classification}
		\label{fig:CiteDWSINC}
	\end{subfigure}
	\quad 
	\begin{subfigure}[b]{0.23\textwidth}
		\includegraphics[width=\textwidth]{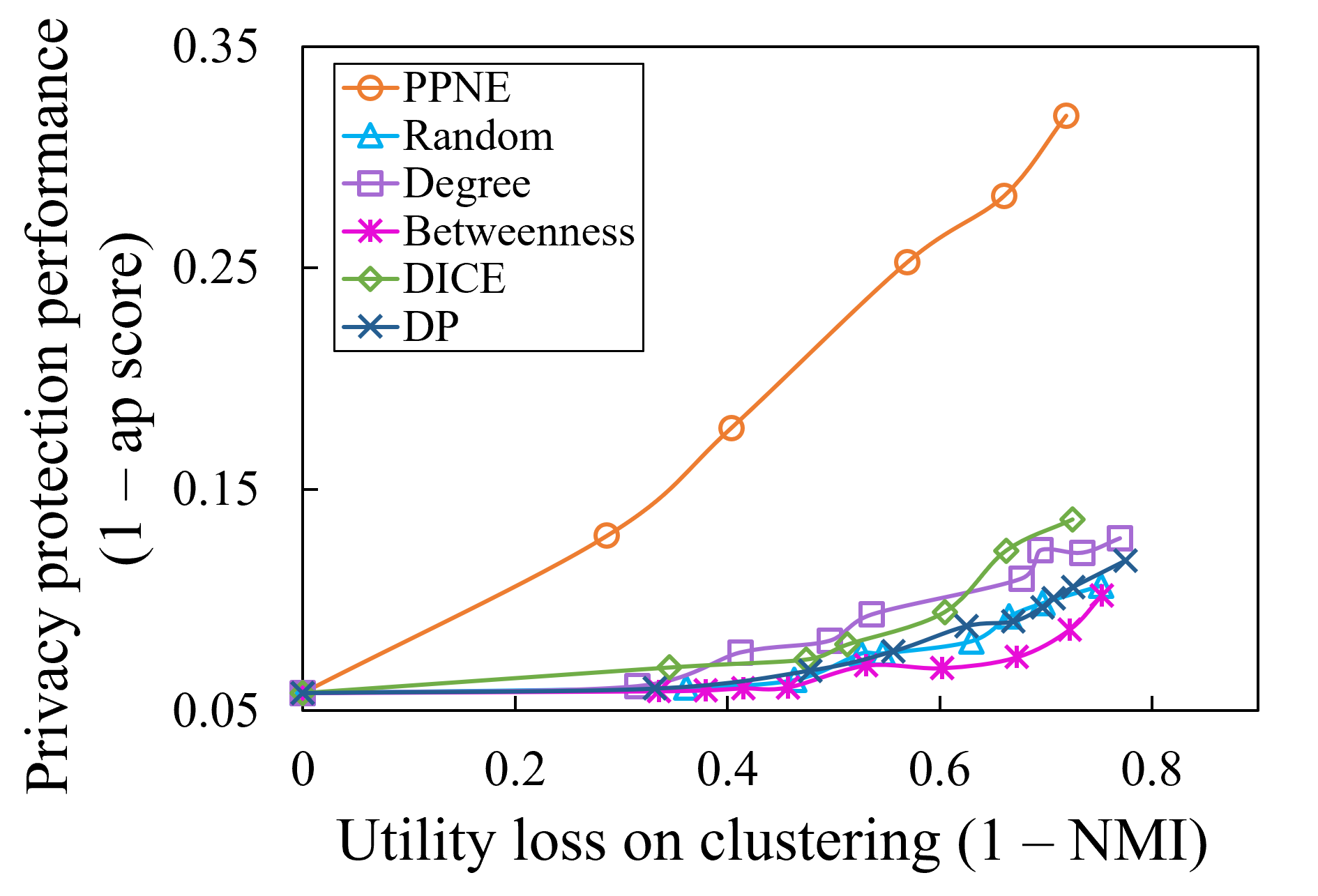}
		\caption{Scenario I on Clustering}
		\label{fig:CiteDWSICL}
	\end{subfigure}
	\quad 
	\begin{subfigure}[b]{0.23\textwidth}
		\includegraphics[width=\textwidth]{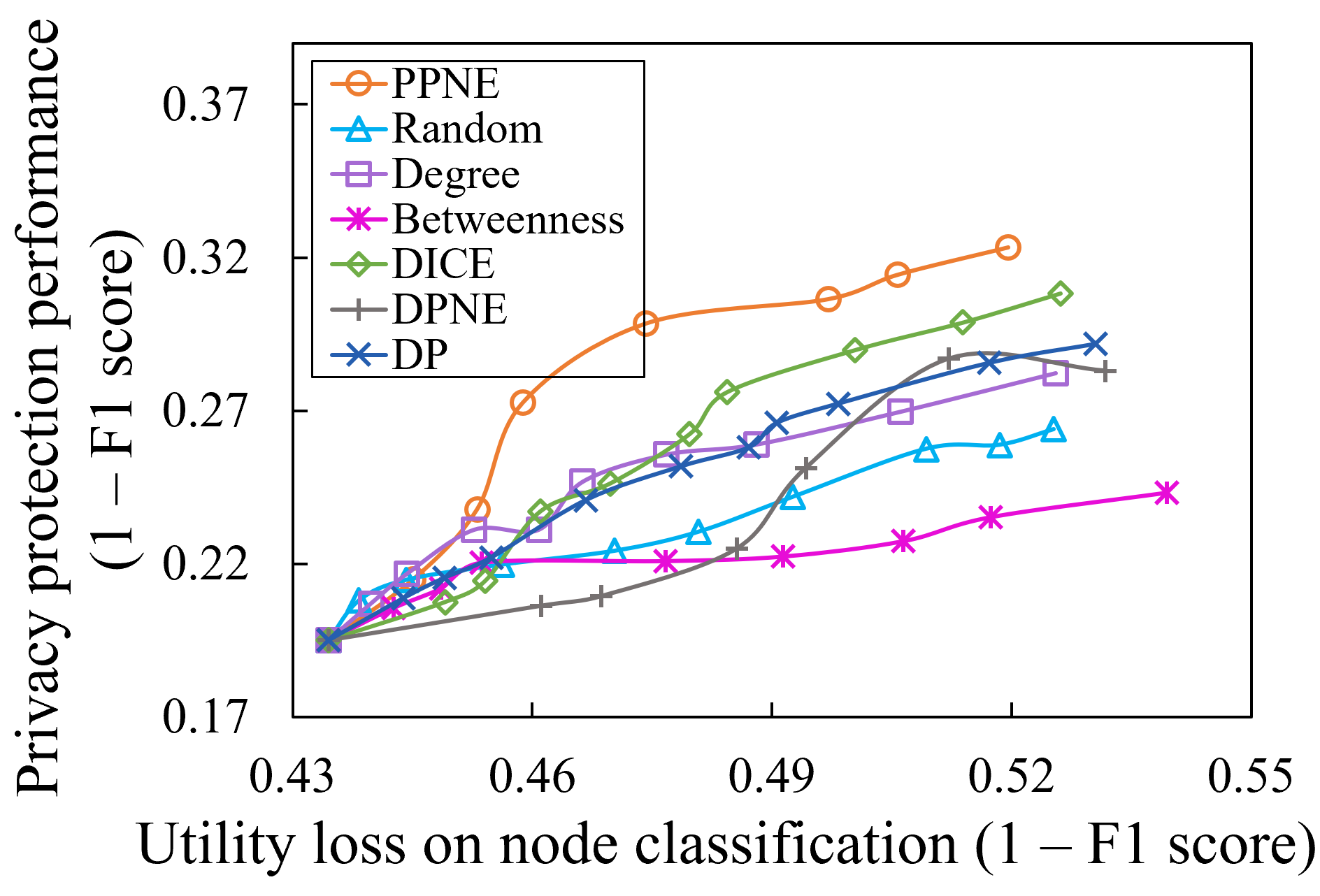}
		\caption{Scenario II on Classification}
		\label{fig:CiteDWSIINC}
	\end{subfigure}
	\quad
	\begin{subfigure}[b]{0.23\textwidth}
		\includegraphics[width=\textwidth]{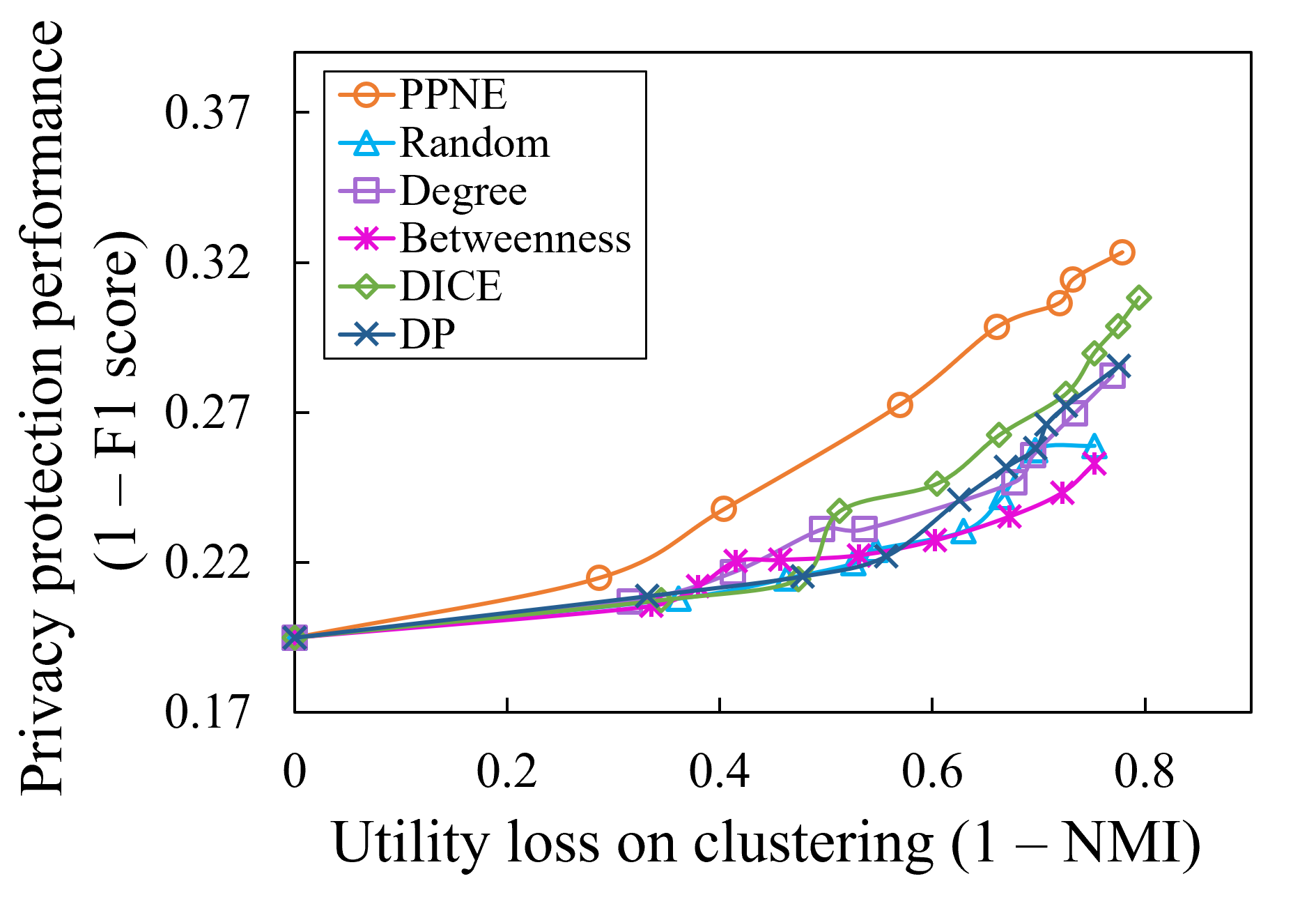}
		\caption{Scenario II on Clustering}
		\label{fig:CiteDWSIICL}
	\end{subfigure}
	\vspace{-1em}
	\caption{Tradeoff on Citeseer: privacy protection on private links and node classification/clustering performance (DeepWalk)}\label{fig:CiteTradeOffDW}
	\vspace{-1em}
\end{figure*}

\begin{figure*}[b]
	\centering
	\begin{subfigure}[b]{0.23\textwidth}
		\includegraphics[width=\textwidth]{figures2/citeseer_line_scenario1_classification.png}
		\caption{Scenario I on Classification}
		\label{fig:CiteLINESINC}
	\end{subfigure}
	\quad 
	\begin{subfigure}[b]{0.23\textwidth}
		\includegraphics[width=\textwidth]{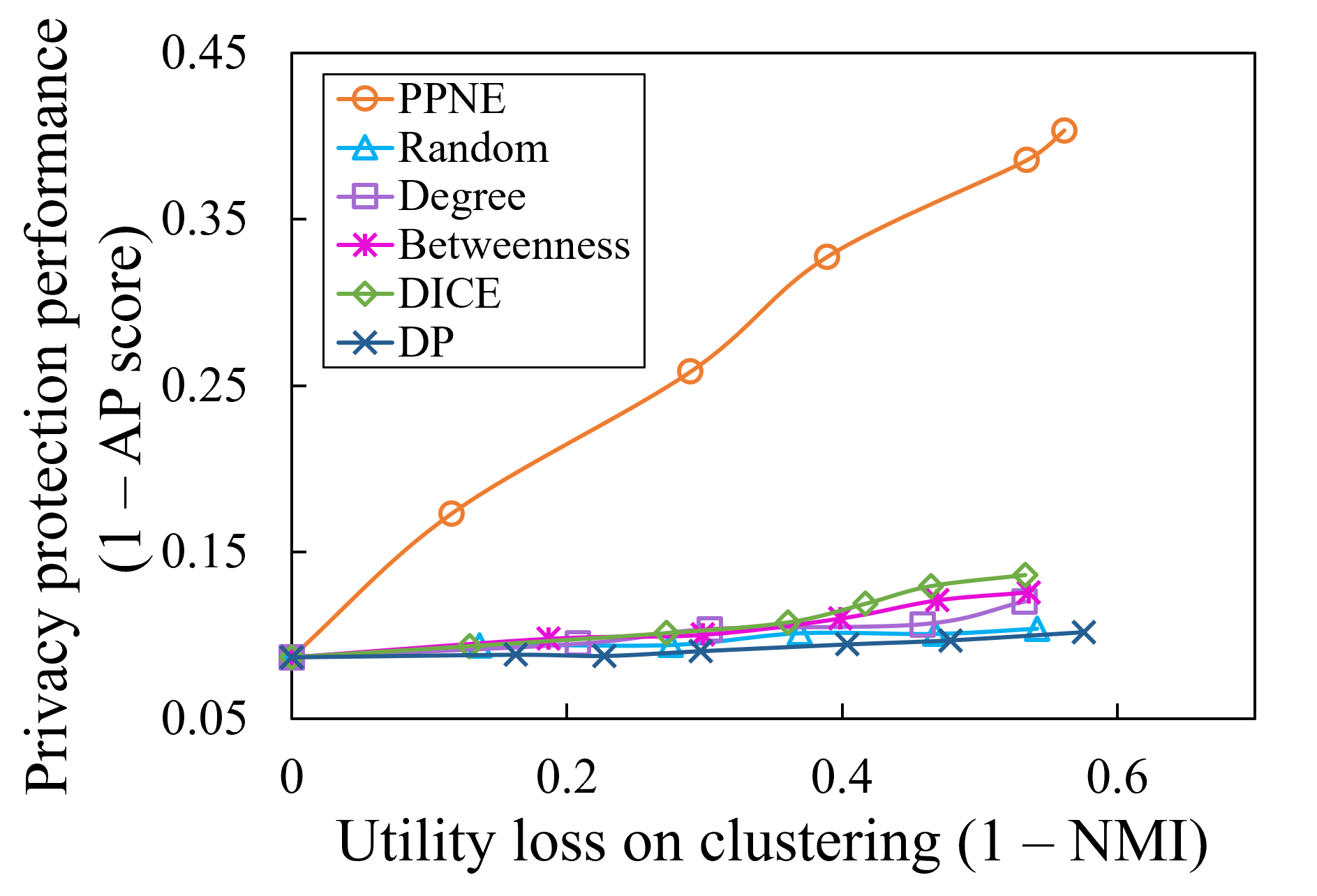}
		\caption{Scenario I on Clustering}
		\label{fig:CiteLINESICL}
	\end{subfigure}
	\quad 
	\begin{subfigure}[b]{0.23\textwidth}
		\includegraphics[width=\textwidth]{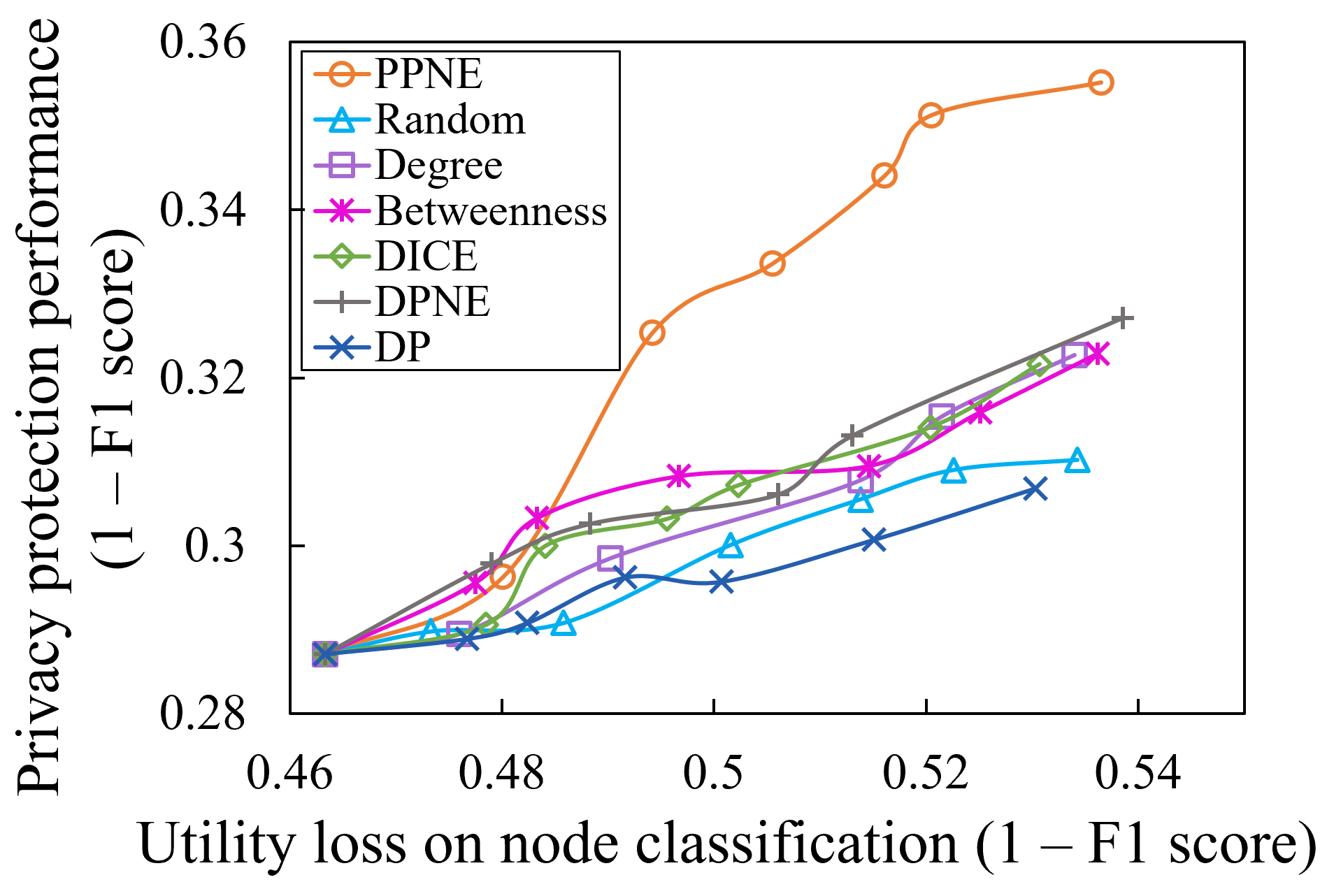}
		\caption{Scenario II on Classification}
		\label{fig:CiteLINESIINC}
	\end{subfigure}
	\quad
	\begin{subfigure}[b]{0.23\textwidth}
		\includegraphics[width=\textwidth]{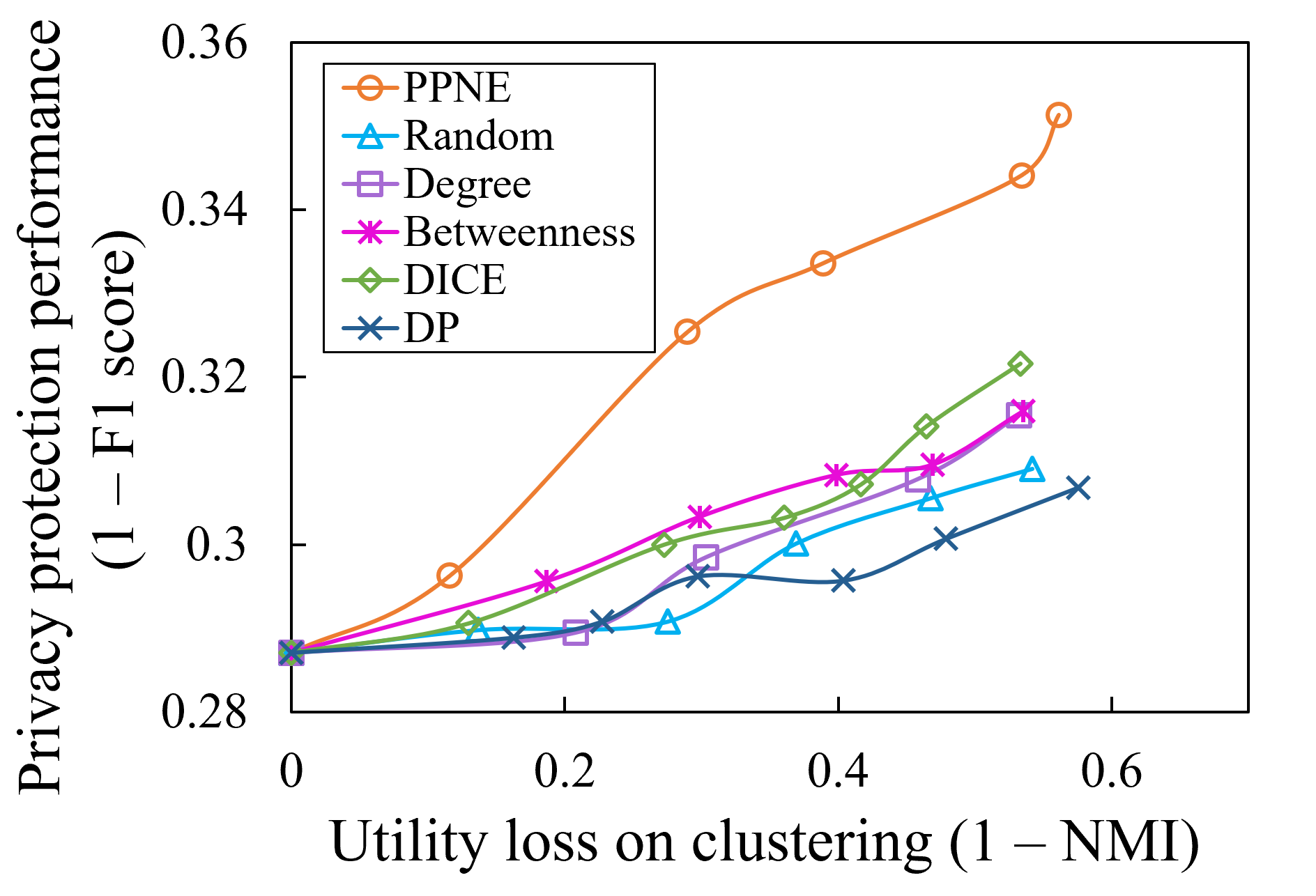}
		\caption{Scenario II on Clustering}
		\label{fig:CiteLINESIICL}
	\end{subfigure}
	\vspace{-1em}
	\caption{Tradeoff on Citeseer: privacy protection on private links and node classification/clustering performance (LINE)}
	\label{fig:CiteTradeOffLINE}
	\vspace{-1em}
\end{figure*}

\begin{figure*}[b]
	\centering
	\begin{subfigure}[b]{0.23\textwidth}
		\includegraphics[width=\textwidth]{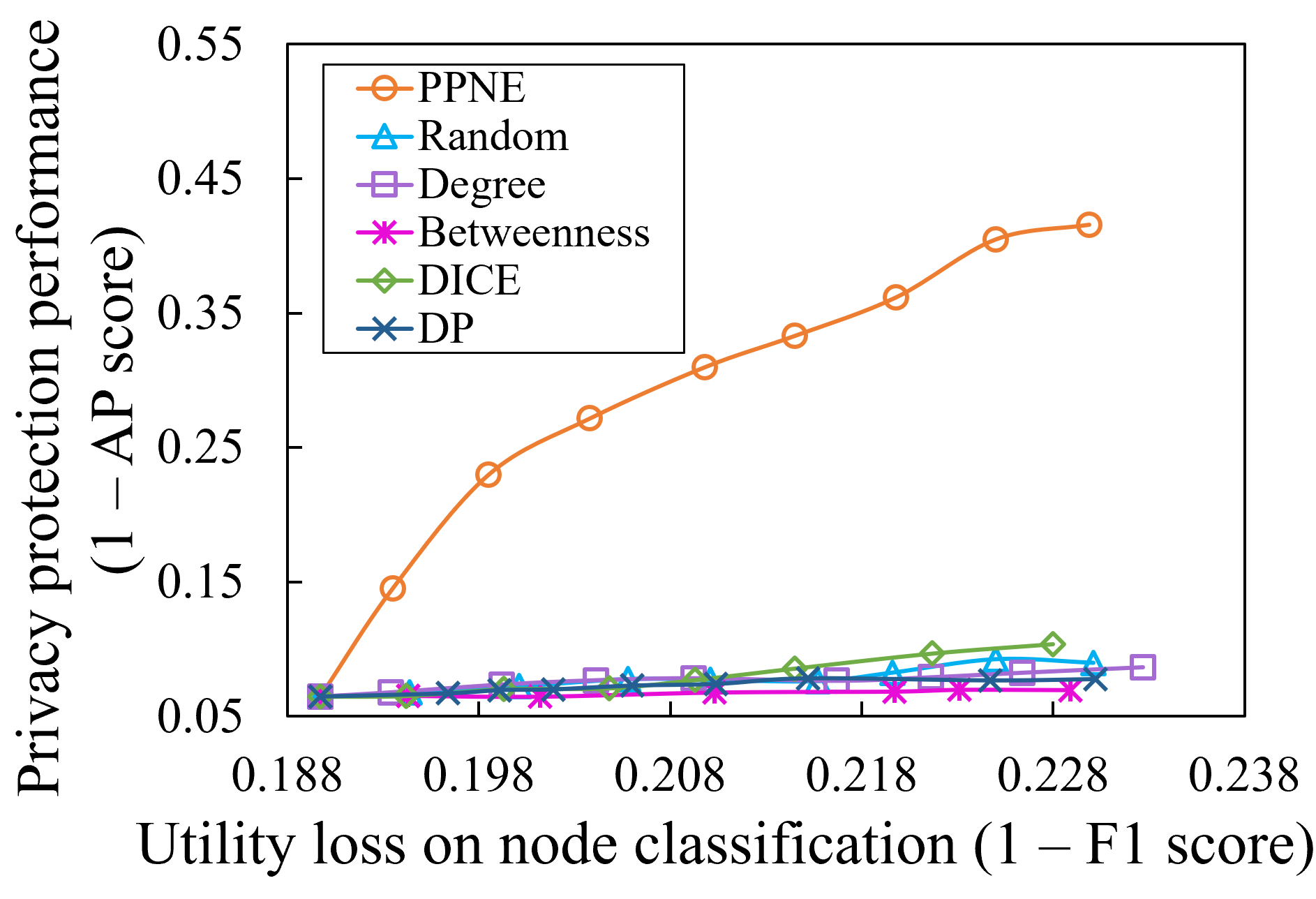}
		\caption{Scenario I on Classification}
		\label{fig:PubDWSINC}
	\end{subfigure}
	\quad 
	\begin{subfigure}[b]{0.23\textwidth}
		\includegraphics[width=\textwidth]{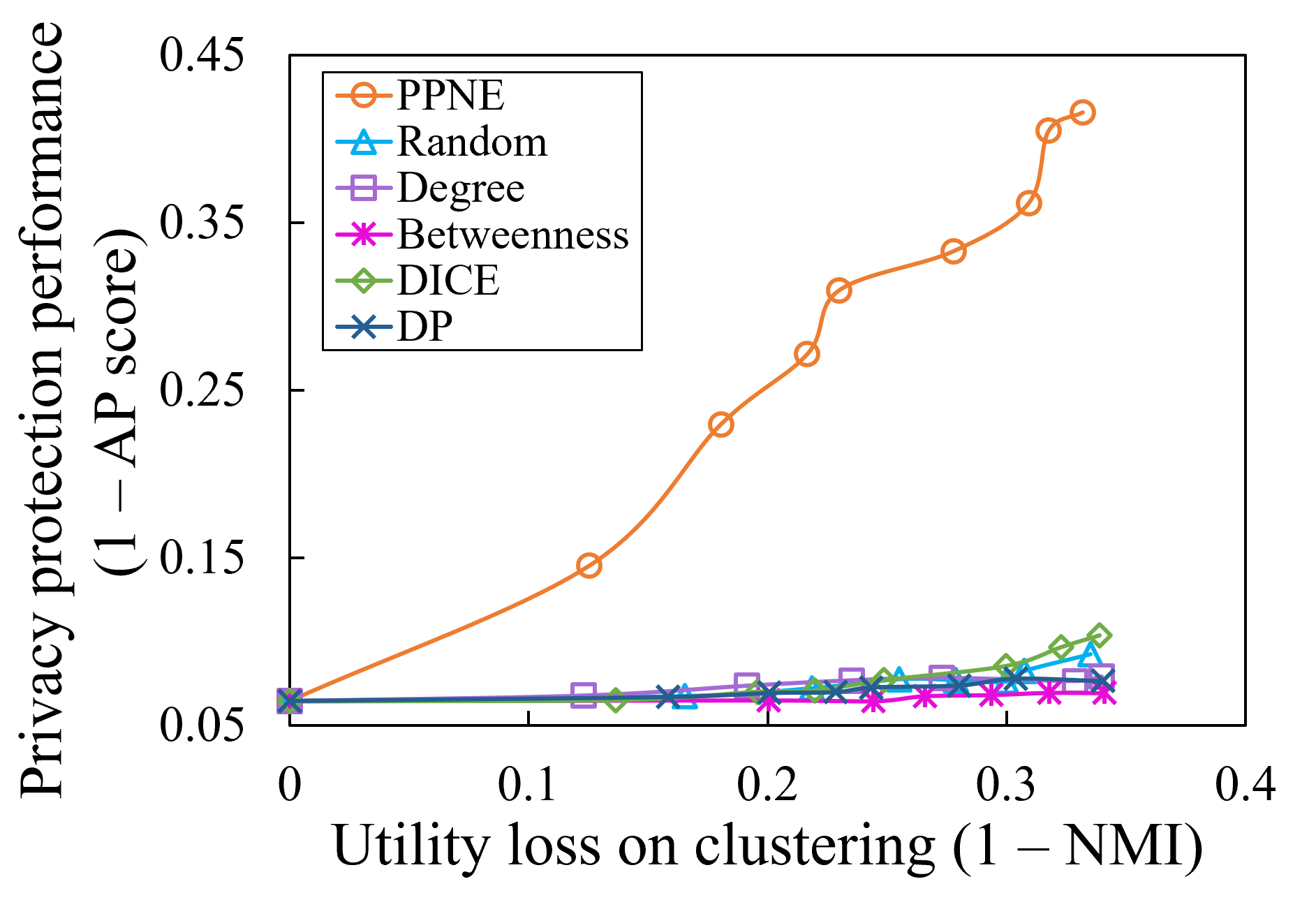}
		\caption{Scenario I on Clustering}
		\label{fig:PubDWSICL}
	\end{subfigure}
	\quad 
	\begin{subfigure}[b]{0.23\textwidth}
		\includegraphics[width=\textwidth]{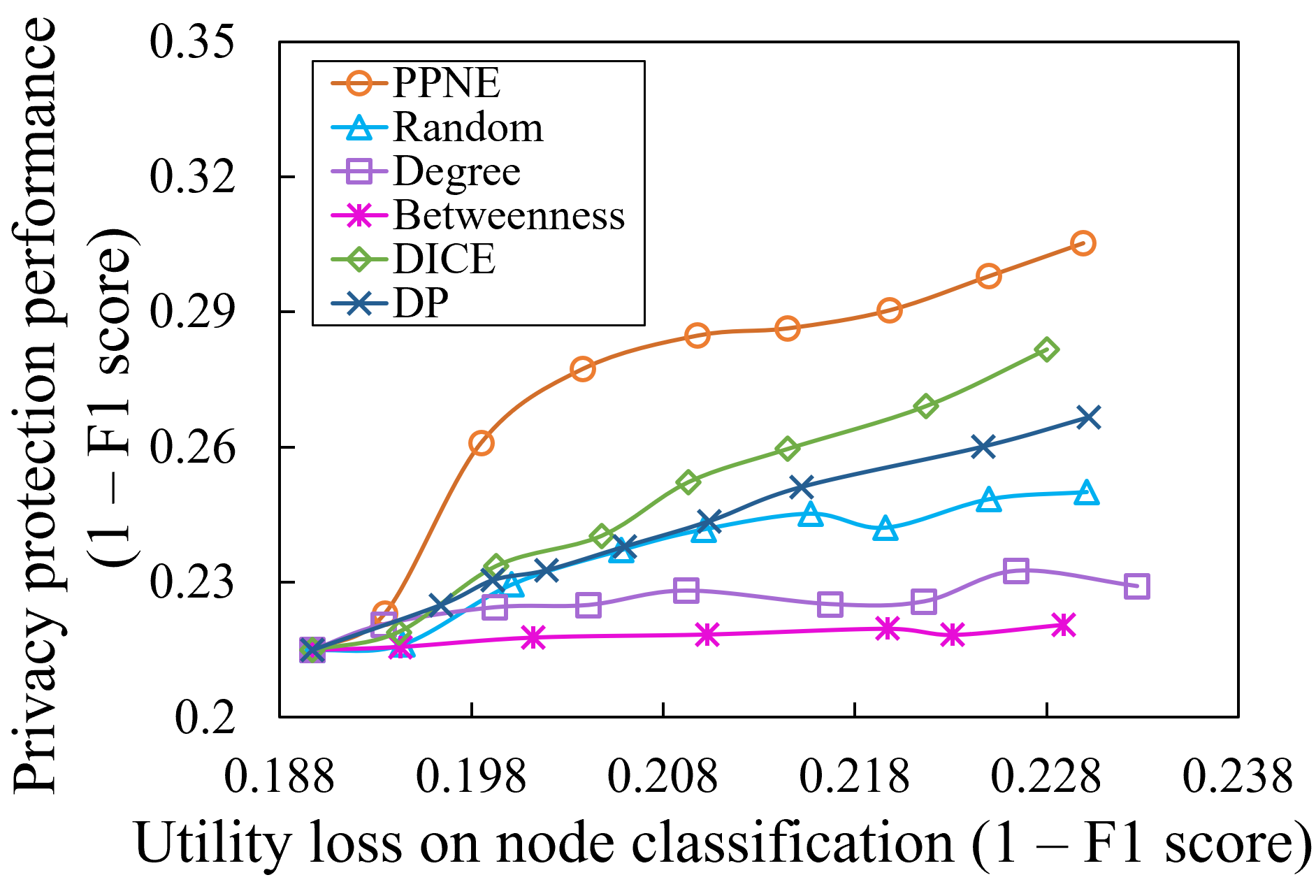}
		\caption{Scenario II on Classification}
		\label{fig:PubDWSIINC}
	\end{subfigure}
	\quad
	\begin{subfigure}[b]{0.23\textwidth}
		\includegraphics[width=\textwidth]{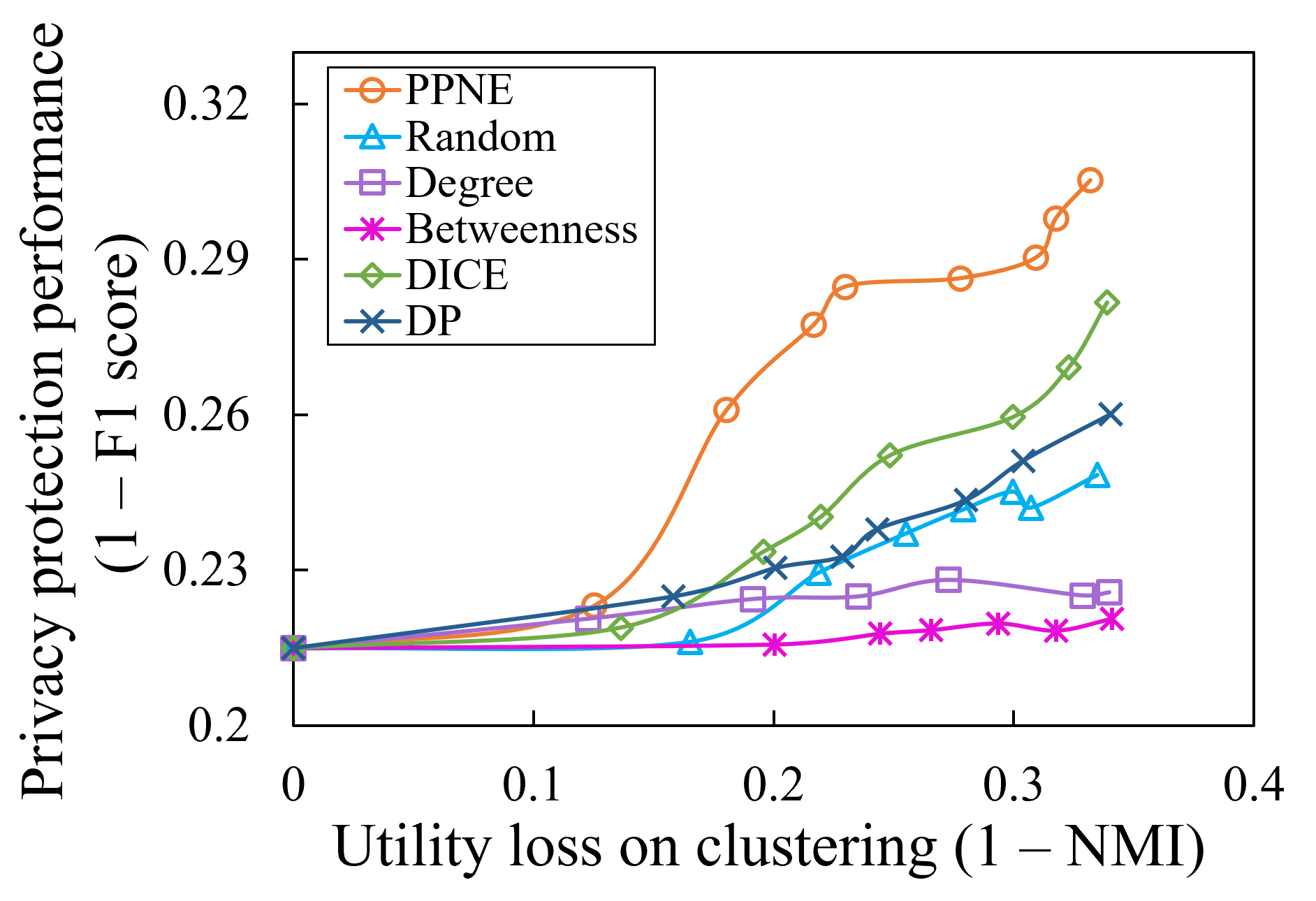}
		\caption{Scenario II on Clustering}
		\label{fig:PubDWSIICL}
	\end{subfigure}
	\vspace{-1em}
	\caption{Tradeoff on PubMed: privacy protection on private links and node classification/clustering performance (DeepWalk)}\label{fig:PubTradeOffDW}
	\vspace{-1em}
\end{figure*}

\begin{figure*}[b]
	\centering
	\begin{subfigure}[b]{0.23\textwidth}
		\includegraphics[width=\textwidth]{figures2/pubmed_line_scenario1_classification.png}
		\caption{Scenario I on Classification}
		\label{fig:PubLINESINC}
	\end{subfigure}
	\quad 
	\begin{subfigure}[b]{0.23\textwidth}
		\includegraphics[width=\textwidth]{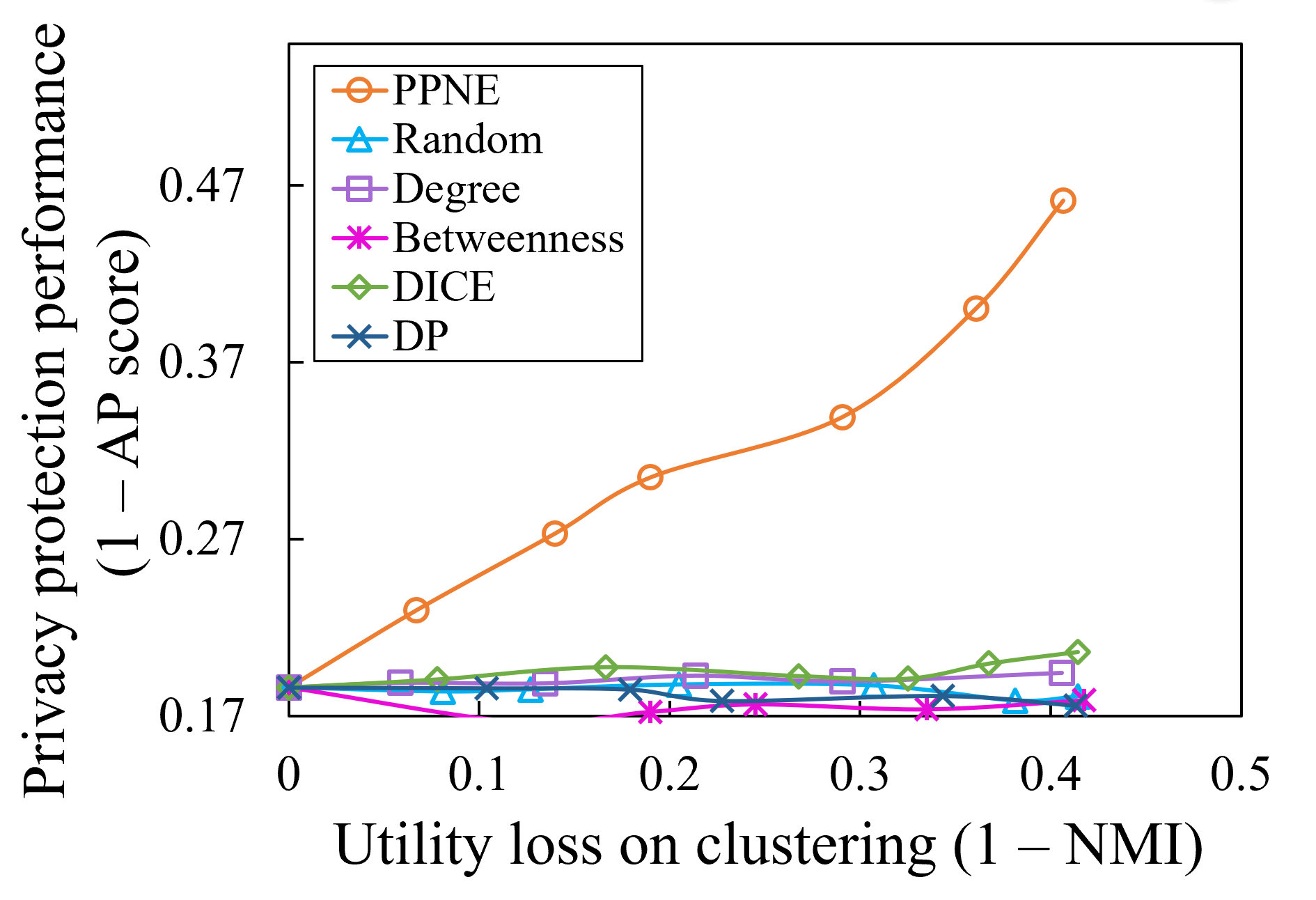}
		\caption{Scenario I on Clustering}
		\label{fig:PubLINESICL}
	\end{subfigure}
	\quad 
	\begin{subfigure}[b]{0.23\textwidth}
		\includegraphics[width=\textwidth]{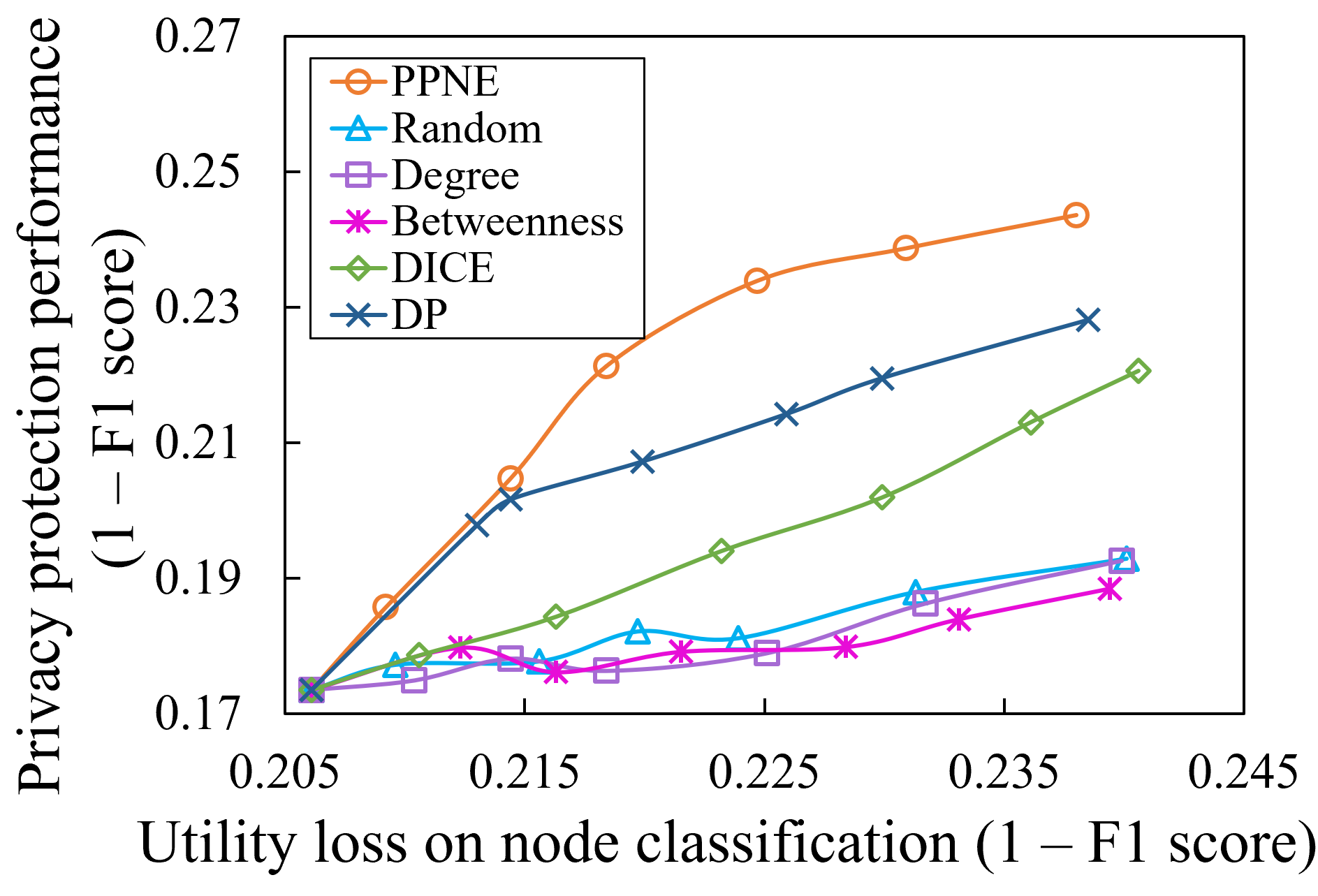}
		\caption{Scenario II on Classification}
		\label{fig:PubLINESIINC}
	\end{subfigure}
	\quad
	\begin{subfigure}[b]{0.23\textwidth}
		\includegraphics[width=\textwidth]{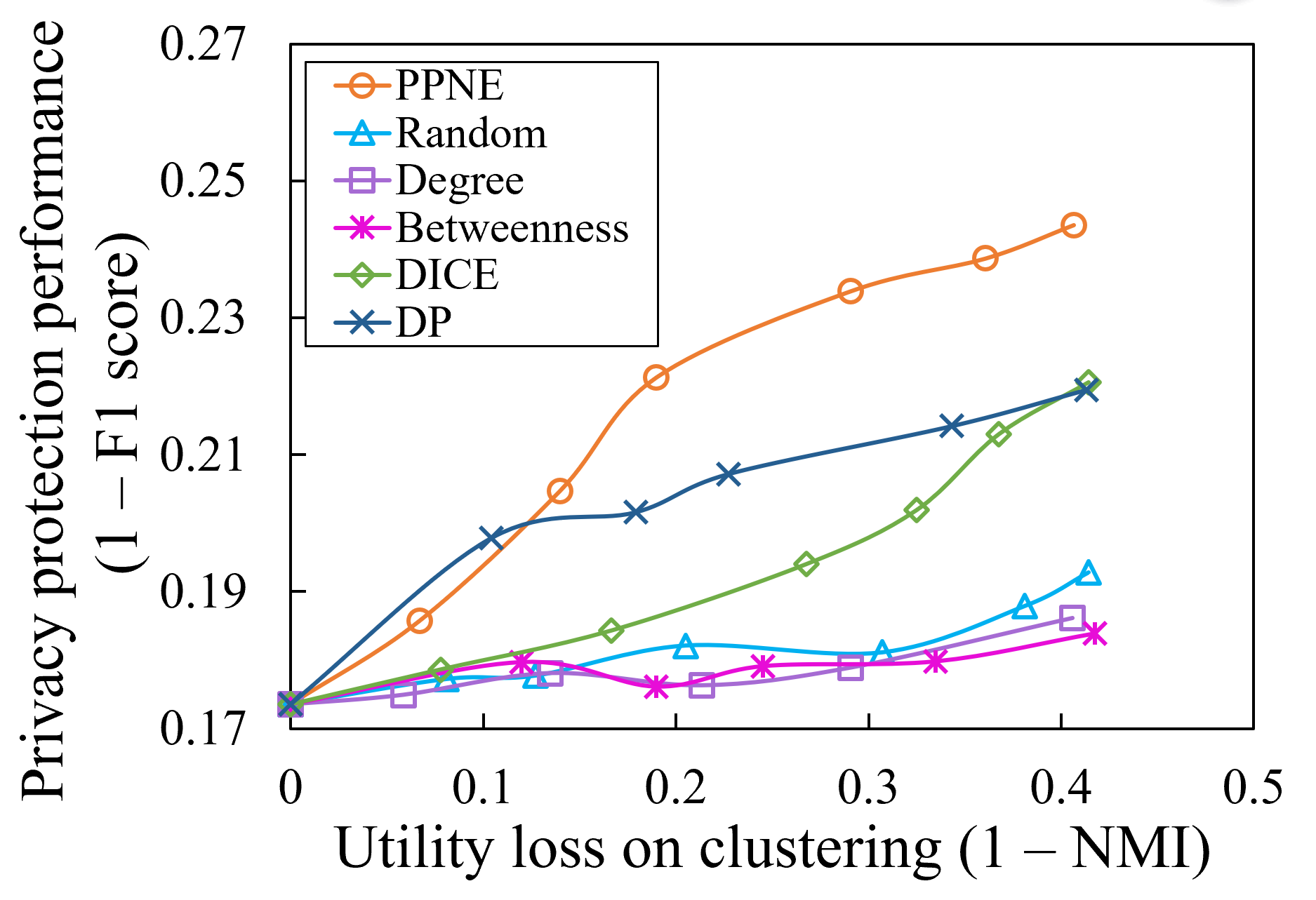}
		\caption{Scenario II on Clustering}
		\label{fig:PubLINESIICL}
	\end{subfigure}
	\vspace{-1em}
	\caption{Tradeoff on PubMed: privacy protection on private links and node classification/clustering performance (LINE)}
	\label{fig:PubTradeOffLINE}
	\vspace{-1em}
\end{figure*} 

\end{document}